\documentclass[leqno]{article}
\usepackage{amssymb,amsmath,amsfonts,amscd,verbatim}
\usepackage{color,graphicx,caption,subcaption}
\usepackage{algorithm,algorithmic}
\usepackage{upgreek,comment,url,comment}
\usepackage{float,units}
\usepackage{hyperref}
\usepackage{xcolor}
\usepackage[margin=1in]{geometry}
\graphicspath{ {Figures/} }

\input{./Definitions}

\usepackage[
backend=biber,
style=numeric,
]{biblatex}

\title{\textbf{Stochastic Rounding Implicitly Regularizes Tall-and-Thin Matrices}}
\author{Gregory Dexter\footnote{Department of Computer Science, Purdue University, West Lafayette, IN 47907, USA, \texttt{\{gdexter, cboutsik, ma856, pdrineas\}@purdue.edu}} \and
Christos Boutsikas\footnotemark[1]\and
Linkai Ma\footnotemark[1]\and
Ilse C.F. Ipsen\footnote{Department of Mathematics, North Carolina State University, Raleigh, NC 27695-8205, USA,
\texttt{ipsen@ncsu.edu}}\and
Petros Drineas\footnotemark[1]
}

\begin{document}
\maketitle

\begin{abstract}
Motivated by the popularity of stochastic rounding in the context of machine learning  and the training of large-scale deep neural network models, 
we consider stochastic nearness rounding of real matrices
$\Ab$ with many more rows than columns.
We provide novel theoretical evidence, supported by extensive experimental evaluation that, with high probability, the smallest singular value 
of a stochastically rounded matrix is well 
bounded away from zero -- regardless of how close
$\Ab$ is to being rank deficient and even if $\Ab$ is rank-deficient. In other words, stochastic rounding  \textit{implicitly regularizes} tall and skinny matrices $\Ab$ so that the
rounded version has full column rank.
Our proofs leverage powerful results in random matrix theory, and the 
idea that stochastic rounding errors do not concentrate
in low-dimensional column spaces.
\end{abstract}

\section{Introduction}

Stochastic Rounding (SR), proposed over 70 years ago, is a probabilistic approach to rounding. According to~\cite{Croci2022}, the earliest proposal for SR appeared in a one-paragraph abstract of a communication presented by Forsythe~\cite{Forsythe1950} in 1950 at the 52nd meeting of the \textit{American Mathematical Society}, in the context of reducing the accumulation of round-off errors in solving systems of ordinary differential equations. Relatedly, the idea of modelling rounding errors as random variables to handle imprecise data in exact arithmetic goes back to 1949 and the work of von Neumann and Goldstine~\cite{von1947numerical}. 

Despite its illustrious beginnings, stochastic rounding has been largely overlooked by the numerical analysis community. Over the past few years, SR has enjoyed a resurgence in popularity, mainly due to the increasing interest for low-precision floating-point arithmetic in the context of machine learning applications and the training of large-scale deep neural network models. Currently, major chip designers own numerous SR-related patents, which seems to indicate that we might soon reach an inflection point for a wider adoption of SR in hardware and software. A non-exhaustive list includes  GraphCore IPUs that support stochastic rounding in binary32 and binary16~\cite{P1}; the Loihi Chip~\cite{davies2018loihi}; AMD~\cite{AMDpatent}; NVDIA~\cite{NVDIApatent}; IBM~\cite{IBM1patent, IBM2patent}; VIA Technologies~\cite{VIApatent}; DensBits Technologies~\cite{DensBitspatent}; and GSI Technology~\cite{GSIpatent}. A detailed discussion of the history of SR and probabilistic error analysis, as well as devices and patents for SR can be found in \cite{Croci2022}.

Recall that, given a number $x \in \mathbb{R}$ and a finite set of numbers $\Fcal \subset \R$, rounding refers to the process of matching $x$ to a number $\tilde{x} \in \Fcal$. This can be done deterministically or \textit{stochastically}: a common modality for stochastic rounding is to round with a probability that depends on the distance of $x$ from the two 
points that enclose it in $\Fcal$. 
For example, if $\Fcal = \{0,1\}$, the value $x=0.7$ is rounded to $\tilde{x}=1$ with probability .7 and to $\tilde{x}=0$ with probability .3.

This version of SR is sometimes called \textit{SR-nearness} or \textit{mode-1 SR}, and is an unbiased estimator. Of particular interest is the case where $\Fcal$ is the set of \textit{normalized floating point} numbers.

\subsection{Our results} Let 
$\Ab\in\R^{n\times d}$ be a tall-and-thin matrix with $n\gg d$. We present novel theoretical evidence, supported by extensive experimental evaluation, that guarantees with high probability that, after SR, the smallest singular value of the rounded matrix \textit{is bounded away from zero}. This holds regardless of how close to rank-deficient $\Ab$ might be, or even if $\Ab$ is rank-deficient, assuming that the rounding process has access to enough randomness (see eqn.~\eqref{eq:maineq} for a precise definition). If the stochastically rounded $\Ab$ were to be used in a downstream regression or classification problem, this is akin to saying that SR~\textit{implicitly regularizes}~$\Ab$. Such regularization effects are often beneficial in downstream machine learning algorithms and, in particular, in training Deep Neural Network (DNN) models and Large Language Models (LLMs)~\cite{gupta2015deep,wang2018training}. Thus, SR could serve as an implicit regularizer in modern machine learning applications, and might bypass the need for explicit regularization. The references in \cite[Section 8.2]{Croci2022} point to a long list of machine learning applications that could benefit from properties of SR.

To give a taste of our results, let $\Fcal$ be the set of normalized floating point numbers. Applying SR entry-wise to
$\Ab\in\R^{n\times d}$ gives the stochastically rounded
version $\Abtil \in \Fcal^{n \times d}$. 

%
%
Our main result, Theorem~\ref{lemma:general_rounding}, combined with (\ref{eqn:pdrcal1}) proves a lower bound for the smallest singular value of $\Abtil$. 
Formally, let $\sigma_d(\cdot)$ denote the smallest singular value of a $n\times d$ matrix where $n\geq d$; let $\beta$ be the basis and $p$ be the working precision of the floating point representation; and assume for simplicity for exposition that all entries of $\Ab$ are in the interval $[-1,1]$. (See Section~\ref{section:main_result} for the general case where the entries of $\Ab$ are arbitrary.) We prove that SR guarantees, with high probability, the following absolute bound for the smallest singular value of the stochastically rounded~$\Ab$,
\begin{align}\label{eq:maineq}
    \sigma_d(\Abtil) \geq \beta^{1-p}\sqrt{n}\left(\sqrt{\nu} - \varepsilon_{n,d}\right).
\end{align}
Here $0\leq \nu\leq 1$ is the \textit{minimum normalized variance} of the stochastic rounding process over all columns of $\Ab$ (defined in Section~\ref{section:main_result}), and $\epsilon_{n,d}$ captures \textit{lower-order} terms that depend only on the dimensions $n$ and $d$ of~$\Ab$.

We discuss the parameters to interpret the above bound. 
\begin{enumerate}
\item As the `tall' dimension $n$ of the matrices $\Ab$ and $\Abtil$ grows, the smallest singular value of the rounded matrix 
$\Abtil$ increases. This is because the columns of~$\Abtil$ have more opportunity to be linearly independent, as they become longer.

\item The parameter $\nu$, formally defined in Section~\ref{section:main_result}, captures the stochasticity affecting the rounding of the entries of $\Ab$.  

To intuitively understand the importance of~$\nu$, consider the special case when~$\Ab$ consists of two identical columns whose entries are elements of $\Fcal$. Any rounding process, including SR, keeps the matrix intact, so that $\Abtil=\Ab$ has two identical columns. Therefore, the smallest singular values of $\Ab$ and $\Abtil$ are equal to zero, since both matrices are rank-deficient. In that case, $\nu=0$, since there is no flexibility in the rounding process. 

Indeed, SR is most powerful when the two points in $\Fcal$ that enclose the entry to be rounded have meaningful probabilities associated with them. 

\item The parameter $\epsilon_{n,d}$ captures \textit{lower-order} terms that depend only on the dimensions $n$ and $d$ of the input matrix. Corollary~\ref{cor:general_rounding} states that if 
$\Ab\in\R^{n\times d}$ is sufficiently tall and thin,
that is $d = \smallO(\left(\nicefrac{n}{\log{n}}\right)^{\nicefrac{1}{4}})$, then
\begin{align*}
    \lim_{n \rightarrow \infty}\varepsilon_{n,d} = 0. 
\end{align*}
Thus, we can drop $\varepsilon_{n,d}$ from the bound (\ref{eq:maineq}), which gives
\begin{align}\label{eq:maineqapprox}
    \sigma_d(\Abtil) \gtrsim \beta^{1-p}\sqrt{n \nu}.
\end{align}
%
This estimate is strongly supported by our empirical evaluations in Section~\ref{exps}, 
textit{and} essentially matches our lower bound in Section~\ref{s_lower}. We conjecture that~(\ref{eq:maineqapprox}) characterizes the true behavior of SR on essentially all tall-and-thin matrices.  
\end{enumerate}
%
%
%

An interesting aspect of (\ref{eq:maineq}) and~(\ref{eq:maineqapprox}) is that they \textit{do not} depend on the closeness of $\Ab$ to rank-deficiency. Thus, SR guarantees that the stochastically rounded matrix~$\Abtil$ invariably has its smallest singular value bounded away from zero, thus has full column rank.


Our proof techniques build upon results from Random Matrix Theory (RMT). Of particular importance is Theorem~\ref{thm:inhomogeneous_anticoncentration}, which first appeared in \cite{dumitriu2022extreme} and bounds the smallest singular value of matrices whose entries are independent but not identically distributed random variables. Our proof first decomposes the matrix of rounding errors into two components, and then bounds the smallest singular value of the first component via RMT, and 
the norm of the second component with a scalar
concentration inequality. The idea is that there is no concentration of error in low-dimensional subspaces. 

The paper is organized as follows: Section~\ref{sec_back} presents notation and basic background, including stochastic rounding; Section~\ref{s_prior} discusses prior work; Section~\ref{s_bounds} presents and discusses our bounds; Section~\ref{exps} presents our experimental evaluations; Section~\ref{section:future_work} discusses future research directions;
and Appendix~\ref{app_G} presents a singular value bound for Gaussian perturbations.

\section{Background} \label{sec_back}
We define notation in Section~\ref{sec_not}; review stochastic rounding in 
Section~\ref{s_setup}; bound the deviation of
a sum random variables from its expectation in Section~\ref{s_Hoeffding};
and recall the union bound for probabilities in Section~\ref{sec_union}.

\subsection{Notation}\label{sec_not}

We use bold uppercase letters to denote matrices and bold lower-case letters to denote vectors. 
The singular values of a matrix $\Ab \in \mathbb{R}^{n \times d}$ with $n \geq d$ are denoted by $\sigma_1(\Ab)\geq \cdots \geq\sigma_d(\Ab)\geq 0$. We use standard notation for matrix and vector norms,
and denote the natural logarithm of $n$ by $\log n$.

The expectation of  a random variable $X$ 
is denoted by $\EE[X]$ and its variance by $\Var[X]$. 
The probability of an event $\cal E$ is $0\leq \Pr[{\cal E}] \leq 1$. The overbar in $\bar{\cal E}$ represents the complement of 
the event~$\cal E$. Recall that $\Pr[{\cal E}] \leq p$ is equivalent
to $\Pr[\bar{\cal E}] \geq 1-p$.

%
The statement $f(n) = \smallO(g(n))$ means $$\lim_{n \rightarrow\infty} \nicefrac{f(n)}{g(n)} =0.$$ 

\subsection{Stochastic rounding and its properties}\label{s_setup}

We review the stochastic rounding model in~\cite{Arar2022}. 
%
%
%
%

Let $\Fcal \subset \R$ be a fixed, finite set of numbers. 
For a number real $x \in [\min \Fcal, \max \Fcal]$, we
represent the enclosing numbers in $\Fcal$ by
\begin{align}\label{eqn:floorceil}
\nceil{x} = \min \{y \in \Fcal: y \geq x\} \quad \text{and}\quad \nfloor{x} = \max \{y \in \Fcal: y \leq x\}.    
\end{align}
For $\nceil{x}\neq \nfloor{x}$, \textit{SR-nearness} of $x \in \R$ is defined as 
\begin{align*}
\roundfunc(x) = \begin{cases}\nceil{x}& 
\text{with probability}\  \frac{x - \nfloor{x}}{\nceil{x} - \nfloor{x}},\\ 
\nfloor{x} & \text{otherwise}.
\end{cases}
\end{align*}
If $\nfloor{x} = \nceil{x}$, then $\roundfunc(x) = x$. SR-nearness produces an unbiased estimator of $x$, namely
\begin{align}
\EE[\roundfunc(x)] = x.
\end{align}
The generalization of SR-nearness to matrices is immediate. If $\Ab \in \R^{n \times d}$, then SR-nearness produces the random matrix $\Abtil \in \R^{n \times d}$ with elements 
\begin{align}
 \Abtil_{ij} = \roundfunc(\Ab_{ij}), \qquad 1\leq i\leq n, \ 1\leq j\leq d.
\end{align}
The random matrix of absolute SR rounding errors is 
\begin{align}
\Eb \equiv \Abtil - \Ab.
\end{align}
The entries of $\Eb$ are independent random variables, and the 
matrix-valued expectation equals $\EE[\Eb] = \zero$. 

\subsection{Normalized floating point numbers} 
In the context of a basis $\beta$ and working precision $p$,
a \textit{normalized} floating point number $x$ can be 
uniquely represented as \cite{8766229}
\begin{align*}
 x = s \cdot m\cdot\beta^{e-p},   
\end{align*}
where $s = \pm 1$ is the sign, $e$ is the exponent, and the \textit{significand} $m$ is an integer in the interval 
$$\beta^{p-1}\leq m < \beta^p.$$ 

With $\Fcal$ representing the set of normalized floating point numbers, the rounding model in Section~\ref{s_setup} 
is exactly the SR-nearness model  from~\cite{Arar2022}, and
as in~\cite{Arar2022}, we ignore numerical overflow and underflow.

Suppose the real number $x \in [\min \Fcal, \max \Fcal]$
is not an element of $\Fcal$. According to ~(\ref{eqn:floorceil}),
$x$ is enclosed by the successive floating point numbers $\nfloor{x}$ and $\nceil{x}$.
%
%
As in deterministic floating point arithmetic,
the SR floating point version 
is either $\nfloor{x}$ or $\nceil{x}$, with an error of\footnote{Compared to SR-nearness, 
the bound for the 
IEEE-754 RN mode (round-to-nearest, ties to even) is tighter by a factor of $\nicefrac{1}{2}$.}
\begin{align}\label{eqn:pd51a}
    \max\{|x-\nfloor{x}|,|x-\nceil{x}|\} \leq \beta^{1-p} |x|.
\end{align}
The absolute distance between the two enclosing floating numbers
equals~\cite[Section II.A, Figure 1]{Arar2022}
\begin{align}\label{eqn:pd52a}
    \nceil{x} - \nfloor{x} = \beta^{e-p}.
\end{align}

\subsection{Bounding the sum of random variables}\label{s_Hoeffding}
The following well-known inequality by W. Hoeffding bounds the deviation of a sum of $n$ independent random variables from its expectation. 
The slight restatement below is adapted to our context.
\begin{theorem}[Theorem 2 in \cite{Hoeff63}]\label{thm:hoeffding}
    Let $X_1,\ldots,X_n$ be independent random variables with $m_i\leq X_i \leq M_i$, $1\leq i\leq n$. Then, for any $t > 0$, 
    \begin{align*}
        \PP\left(\Big|\sum_{i=1}^n (X_i - \EE[X_i])\Big| \geq t\right)
        \leq 2\exp\left(\frac{-2t^2}{\sum_{i=1}^n (M_i - m_i)^2}\right).
    \end{align*}
\end{theorem}

\subsection{The union bound}\label{sec_union}

The following union bound 
states that the probability of the union of 
$k$ events ${\cal E}_i$, $1\leq i \leq k$, 
is bounded above by the sum of the individual probabilities, i.e., 
\begin{align*}
   \Pr[{\cal E}_1 \cup \cdots \cup{\cal E}_k] \leq \sum_{i=1}^k \Pr[{\cal E}_i]. 
\end{align*}

%


\section{Prior work}\label{s_prior}
%
%
There exists a large body of prior work discussing the implementation and error analysis of stochastic rounding in the context of floating point arithmetic. In the 1950s, Forsythe~\cite{forsythe1959reprint} modeled round-off errors as random variables. A few years later, Hull and Swenson~\cite{hull1966tests} presented probabilistic models for round-off errors and concluded that such models are, in general, very good in theory and in practice. The focus of our work is the SR-nearness mode, which first appeared in~\cite{parker1997monte}, while the SR-up-or-down mode was analyzed in~\cite{vignes2004discrete}. Other work analyzed SR for the heat equation~\cite{croci2023effects}; proved that SR prevents stagnation~\cite{connolly2021stochastic}; analyzed SR for floating point summation~\cite{hallman2023precision}; and
proposed alternative frameworks to characterize SR errors for sequential summation and Horner's method for evaluating polynomials based on the computation of the variance and Chebyshev's inequality instead of martingales~\cite{Arar2023}. Software emulators of SR include Verificarlo~\cite{denis2015verificarlo}, Verrou~\cite{fevotte2016verrou}, and Cadna \cite{jezequel2008cadna}. Finally,~\cite{Croci2022} presents a survey of error analysis and applications of SR including a more general analysis of SR errors that are not necessarily independent, but only weakly independent.

Our contribution is to demonstrate that SR rounding tends to increase the smallest singular value of tall-and-thin matrices, thus performing implicit regularization when these matrices are used in downstream machine learning and data analysis applications. 
Our work was partially motivated by~\cite{Sankar2006}, where Gaussian elimination without pivoting 
is analyzed in the smoothed complexity model, and the smallest singular value of $\Ab+\Eb$ is bounded from below, 
for a matrix~$\Eb$ whose entries are independent, identically distributed Gaussian random variables. While the proof techniques of~\cite{Sankar2006} are not directly portable to our setting, the motivation is somewhat similar.

Finally, our own prior work~\cite{Boutsikas2024} demonstrates both theoretically and experimentally that perturbing a real matrix $\Ab$ of full column rank can potentially increase the smallest singular values under certain assumptions involving singular
value gaps, thus establishing a qualitative model for the increase in  small singular values after a matrix has been downcast to a lower arithmetic precision. However, the bounds in~\cite{Boutsikas2024} have a different flavor and are not directly comparable the bounds here.

\section{Bounding the smallest singular value of the perturbed matrix}\label{s_bounds}

We present a simple example for the effect of SR-nearness on the smallest singular value in 
Section~\ref{sxn:simpleexample};
derive a singular value bound 
in Section~\ref{section:main_result} for general random errors and in Section~\ref{s_floats} for SR-nearness errors; review the RMT result for our proof 
in Section~\ref{sec_RMT}; and finally present the proof of
Theorem~\ref{lemma:general_rounding}
in Section~\ref{sec_proof} and its tightness
in Section~\ref{s_lower}.

The goal is to quantify how SR-nearness applied to
a tall and skinny matrix $\Ab \in \R^{n \times d}$ with $n \gg d$ increases its smallest singular value. This is in contrast to deterministic rounding, which 
can drive the smallest singular value  to zero. 

A first approach for bounding the smallest singular value $\sigma_d(\Ab+\Eb)$ of the rounded matrix $\Ab+\Eb$
might rely on Weyl's inequality,
\begin{align}\label{eqn:pd0}
\sigma_d(\Ab + \Eb) \geq \sigma_d(\Ab) - \|\Eb\|_2.
\end{align}
However, (\ref{eqn:pd0}) is not informative if $\sigma_d(\Ab)=0$. More generally, (\ref{eqn:pd0})
produces positive bounds only if the smallest singular value of $\Ab$ is larger than $\|\Eb\|_2$. 

This is the reason why we are exploring lower bounds that \textit{do not} depend on the smallest singular value of $\Ab$. We argue that, under appropriate assumptions, the smallest singular value of the rounded matrix $\Ab+\Eb$ cannot be too small, and that its value does not depend on the singular values of the original matrix $\Ab$. 

To this end, we employ Random Matrix Theory (RMT) to derive bounds that \emph{only depend on the distribution of} the entries of $\Eb$. Our lower bounds for 
$\sigma_d(\Ab+\Eb)$ can be positive 
even if $\sigma_d(\Ab) = 0$. Therefore, SR-nearness can \emph{increase} the smallest singular value even in the extreme case of rank-deficient matrices

\subsection{A simple example}\label{sxn:simpleexample}
We illustrate the effect of stochastic rounding on the smallest singular value of $\Ab$, in the special case where $\Ab \in \R^{n \times 2}$ is a rank-one matrix, all of whose entries are equal to $\nicefrac{1}{2}$. 
Suppose we want to round $\Ab$ so as to represent each entry in terms of a single bit, i.e., ${\cal F} = \{0,1\}$. 

A deterministic model that rounds $\nicefrac{1}{2}$ to one
produces a rounded matrix that is rank deficient as well. 
In contrast, SR-nearness sets each entry of $\Abtil$ to zero or one with equal probability. Hence $\rank(\Abtil)=1$ only if the two columns of $\Abtil$ are identical -- an event whose probability becomes exponentially small as $n$ increases. Here is an example of what SR-nearness may look like for an $8 \times 2$ matrix:
\begin{align*}
    \Ab = 
    \begin{bmatrix}
        $\nicefrac{1}{2}$ & $\nicefrac{1}{2}$ \\
        $\nicefrac{1}{2}$ & $\nicefrac{1}{2}$\\
        $\nicefrac{1}{2}$ & $\nicefrac{1}{2}$\\
        $\nicefrac{1}{2}$ & $\nicefrac{1}{2}$\\
        $\nicefrac{1}{2}$ & $\nicefrac{1}{2}$\\
        $\nicefrac{1}{2}$ & $\nicefrac{1}{2}$\\
        $\nicefrac{1}{2}$ & $\nicefrac{1}{2}$\\        
        $\nicefrac{1}{2}$ & $\nicefrac{1}{2}$
    \end{bmatrix}
    \qquad
    \Abtil = 
    \begin{bmatrix}
        1 & 0 \\
        1 & 0\\
        0 & 1\\
        1 & 0\\
        1 & 1\\
        1 & 1\\
        0 & 0\\
        1 & 0\\
    \end{bmatrix}.
\end{align*}
In Appendix~\ref{sxn:app:extra} we prove that the smallest singular value of the $n\times 2$ matrix $\Abtil$ satisfies the following lower bound, with probability at least $0.997$:
\begin{align}
 \sigma_2^2(\Abtil) \geq 0.25 \cdot n - 8\sqrt{n} = \Omega(n) \Rightarrow \sigma_2(\Abtil) \geq \Omega(\sqrt{n}).
\end{align}
%
This illustrates that SR-nearness is highly likely to produce a significant increase in the smallest singular value.





\subsection{General random perturbation} \label{section:main_result}
Our main result in Theorem~\ref{lemma:general_rounding}
bounds the smallest singular value of the rounded matrix $\Abtil=\Ab+\Eb \in \mathbb{R}^{n \times d}$ away from zero, provided: 
(i) there is enough randomness in the perturbation $\Eb$, and (ii) $\Ab$ is sufficiently tall and thin. 

We define the minimum normalized column-wise variance 
$\nu$ of $\Eb$ as follows:
\begin{align}\label{eqn:nupd}
\nu \equiv \frac{1}{n\Rcal^2}\,\min_{1\leq j\leq d}\,\sum_{i=1}^n  \Var\left(\Eb_{ij}\right)\quad \mbox{with}\quad 
\max_{i,j}|\Eb_{ij}|\leq \Rcal.
\end{align}
Clearly, $\Var(\Eb_{ij}) = \Var(\Abtil_{ij})$, which implies that  
$\Var\left(\Eb_{ij}\right) \leq \Rcal^2$. Therefore, 
\begin{align*}
   0\leq \nu\leq 1. 
\end{align*}
%
%
Intuitively, $\nu$ characterizes the amount of randomness in SR-nearness. Our main theorem 
below depends on the minimal column-wise variance of the perturbation $\Eb$.

\begin{theorem}\label{lemma:general_rounding}
Let $\Ab$ and $\Abtil= \Ab+\Eb$
be real $n\times d$ matrices with $n\gg d$.
Here~$\Eb$ models random perturbations with minimal normalized column variance $\nu$ and $\Rcal$ from~(\ref{eqn:nupd}). 

If $n \geq 836$, then
%
with probability at least $1 - \frac{1}{n^c} - \frac{2d^2}{n^2}$, 
\begin{align*}
    \sigma_d(\Abtil) \geq \Rcal\sqrt{n}(\sqrt{\nu} - \varepsilon_{n,d}),
\end{align*}
where
\begin{align}\label{eqn:epsilonnd}
\varepsilon_{n,d} \equiv \sqrt{\frac{d}{n}} + 2d^2\sqrt{\frac{\log n}{n}} + \frac{C(\log n)^{2/3}}{n^{1/30}} \cdot \left(\frac{d}{n}\right)^{\nicefrac{1}{54}},
\end{align}
and $c$ and $C$ are absolute constants\footnote{These constants are unspecified in \cite[Theorem 2.10 and Remark 2.11]{dumitriu2022extreme}.}.
\end{theorem}
The comments below provide intuition for Theorem~\ref{lemma:general_rounding}. 
\begin{enumerate}
\item The lower bound for $\sigma_d(\Abtil)$ is very general; it holds for any random matrix~$\Eb$, regardless of whether it models SR-nearness errors or not. However,
Theorem~\ref{lemma:general_rounding} requires that $\Eb$ merely change entries of $\Ab$ by a small amount, quantified by $\Rcal$. In other words, a small value of $\Rcal$ prevents large changes in individual  entries of $\Ab$, thereby preventing
them from exerting disproportionate influence on the smallest singular value of $\Abtil$.
%
%
\item For the bound to be  positive, we need $\epsilon_{n,d}\leq \sqrt{\nu} \leq 1$. If
\begin{align}\label{eqn:dassumption}
d = \smallO(\left(\nicefrac{n}{\log n}\right)^{\nicefrac{1}{4}}),
\end{align}
then the first two terms of $\varepsilon_{n,d}$ must approach 0 as 
$n\rightarrow\infty$. This is because $d = \smallO(\left(\nicefrac{n}{\log n}\right)^{\nicefrac{1}{4}})$ is equivalent to $\lim_{n \rightarrow\infty} \frac{d^4 \log n}{n} = 0$, which in turn implies 
\begin{align*}
\lim_{n \rightarrow\infty} 2d^2\sqrt{\frac{\log n}{n}}=0 
\qquad\text{and}\qquad
\lim_{n \rightarrow \infty} \sqrt{\frac{d}{n}}=0.
\end{align*}
The third term also approaches 0 as $n\rightarrow\infty$, because
\begin{align*}
(\log n)^{2/3} = \smallO(n^{1/30})\quad \Longleftrightarrow
\quad\lim_{n \rightarrow \infty} \nicefrac{(\log n)^{2/3}}{n^{1/30}}=0.    
\end{align*}
%
%
%
The ratio $\nicefrac{(\log n)^{2/3}}{n^{1/30}}$ goes to zero slowly as $n$ grows. For example, $n$ needs to be larger than $10^{50}$ for this ratio to drop to $\nicefrac{1}{2}$. An important question for future research is the strengthening of Theorem~\ref{lemma:general_rounding} to reduce $\varepsilon_{n,d}$. 

In the Appendix, we prove Theorem~\ref{thm:gaussian} for the special case where the elements $\Eb_{ij}$ are independent, identically distributed Gaussian normal random variables and show that the smallest singular value of $\Abtil$ is again bounded away from zero with high probability. 

However, Theorem~\ref{thm:gaussian} is much sharper than Theorem~\ref{lemma:general_rounding}. For example, even if $n = 900$ and $d = 25$, the bound $\sigma_d(\Abtil) \geq 1$ holds with probability at least $0.98$, and does not need a bound on $\max_{i,j}|\Eb_{ij}|$; it suffices that the variance of $\Eb_{ij}$ is bounded. Hence it might
be possible to significantly reduce $\epsilon_{n,d}$ in Theorem~\ref{lemma:general_rounding} or even eliminate it,
via RMT work on non-Gaussian, non-identically distributed perturbations.
%
%
The experiments in Section~\ref{exps} illustrate that $\varepsilon_{n,d}$ is an artifact of our analysis and not a real concern in practice.

\item The success probability $1 - \frac{1}{n^c} - \frac{2d^2}{n^2}$
approaches 1 as $n\rightarrow \infty$, in spite of the unspecified absolute constant $c$. This is because the  assumption~(\ref{eqn:dassumption}) on $d$ implies 
\begin{align*}
\lim_{n \rightarrow\infty} \nicefrac{1}{n^c} = 0\quad \text{and}\quad \lim_{n \rightarrow\infty} \nicefrac{2d^2}{n^2}=0.
\end{align*}
\end{enumerate}

The above discussion implies the following bound for sufficiently 
tall and skinny matrices.

\begin{corollary}\label{cor:general_rounding}
Under the assumptions of Theorem~\ref{lemma:general_rounding}
if also 
\begin{align*}
d = \smallO(\left(\nicefrac{n}{\log n}\right)^{\nicefrac{1}{4}}),
\end{align*}
then with probability approaching one, 
\begin{align}\label{eqn:pd233}
    \sigma_d(\Abtil) \gtrsim \Rcal\sqrt{n\cdot \nu}.
    %
\end{align}
\end{corollary}
%
%

\subsection{Perturbations from SR-nearness}\label{s_floats} Applying Theorem~\ref{lemma:general_rounding} to SR-nearness
with normalized floating point numbers $\Fcal$ is straight-forward. 

First, we focus on the case where $\Ab_{ij}$ round to normalized numbers in 
$[-1,1]$ closest to zero,
$1\leq i\leq n$, $1\leq j\leq d$.
From (\ref{eqn:pd51a}) follows that the error
in the elements of $\Abtil \in \Fcal^{n \times d}$ is at most 
$\beta^{1-p}$. Thus
\begin{align}\label{eqn:pdrcal1}
    \max_{i,j} |\Ab_{ij}-\Abtil_{ij}| = \max_{i,j} |\Eb_{ij}| \leq \Rcal\equiv \beta^{1-p}|\Ab_{ij}|.
\end{align}
Theorem~\ref{lemma:general_rounding} and Corollary~\ref{cor:general_rounding} apply immediately with 
$\Rcal$ from~(\ref{eqn:pdrcal1}).

Now assume that $\Ab \in \mathbb{R}^{n \times d}$ 
is a general matrix. From (\ref{eqn:pd52a}) follows
that the error in $\Abtil_{ij}$ is at most $\beta^{e_{ij}-p}$, 
where $e_{ij}$ is the exponent for $\Ab_{ij}$, $1\leq i\leq n$, 
$1\leq j\leq d$. Thus
\begin{align}\label{eqn:pdrcal2}
    \max_{i,j} |\Ab_{ij}-\Abtil_{ij}| = \max_{i,j} |\Eb_{ij}| \leq \max_{i,j }\beta^{e_{ij}-p} \leq \Rcal\equiv \beta^{e_{\max}-p},
\end{align}
where $e_{\max} = \max_{i,j}\{e_{ij}\}$. Theorem~\ref{lemma:general_rounding} and Corollary~\ref{cor:general_rounding} immediately apply 
with $\Rcal$ from~(\ref{eqn:pdrcal2}).


\subsection{A Random Matrix Theory bound}\label{sec_RMT}
We present the basis for our proof, which
is a lower bound on the minimum singular value from
\cite[Theorem 2.10 and Remark 2.11]{dumitriu2022extreme}
for matrices whose elements are independent zero-mean random numbers that are not necessarily identically distributed. This latter fact is crucial, since the elements
of $\Eb = \Abtil - \Ab$ are independent but not identically distributed\footnote{To be precise, we adapt results from the arXiv version~\cite{dumitriu2022extreme}, specifically in the regime where $\gamma \to 0$, as the published version~\cite{dumitriu2024extremepub} provides results only for the case when $\gamma = O(1)$. See also~\cite{brailovskaya2024universality} for an overview and results on matrix concentration inequalities that have a similar flavor to the bounds used in our paper.}.

\begin{theorem}[Theorem 2.10 and Remark 2.11 in~\cite{dumitriu2022extreme}]\label{thm:inhomogeneous_anticoncentration}
     Let $\Xb\in\R^{n \times d}$ with $n\geq d$ have independent zero-mean entries, so that  $\EE[\Xb] = \zero$. Suppose there are $q > 0$, $\kappa \geq 1$, and $0<\gamma \leq 1$ such that:
    \begin{align}
        \max_{1\leq i\leq n, 1\leq j\leq d} |\Xb_{ij}| &\leq \nicefrac{1}{q},\label{ass1}\\
        \max_{1\leq i\leq n, 1\leq j\leq d} \EE|\Xb_{ij}|^2 &\leq \nicefrac{\kappa}{n},\label{ass2} \\
        \max_{1\leq j\leq d} \sum_{i=1}^n \EE|\Xb_{ij}|^2 &\leq 1,\label{ass3} \\
        \max_{1\leq i \leq n} \sum_{j=1}^d \EE|\Xb_{ij}|^2 &\leq \gamma,\label{ass4}\\
        \min_{1\leq j\leq d} \sum_{i=1}^n \EE|\Xb_{ij}|^2 &\geq \rhom \geq \sqrt{\gamma}.\label{ass5}
    \end{align}
    If also
    \begin{align}\label{ass6}
        \sqrt{\log n} \leq q \leq n^{\frac{1}{10}}\kappa^{-\frac{1}{9}}\gamma^{-\frac{1}{18}},
    \end{align}
    then, with probability at least $1 - n^{-c}$,
    \begin{align*}
        \sigma_d(\Xb) \geq \sqrt{\rhom} - \sqrt{\gamma} - Cq^{-1/3}(\log{n})^{2/3},
    \end{align*}
    where $c$ and $C$ are absolute constants.
\end{theorem}

Assumptions (\ref{ass1})--(\ref{ass5}) in Theorem~\ref{thm:inhomogeneous_anticoncentration}
require the
following quantities associated with $\Xb$ to be sufficiently small:
the elements, the variances of the elements, 
the maximal variance of the columns,
    and the maximal variance of the rows.
    However, to guarantee enough randomness, the minimal variance of the columns
    is not allowed to be too small.
    
Assumption (\ref{ass6}) implies
$q\geq \sqrt{\log n}\geq 3$, 
so that in (\ref{ass1}) all elements in $\Xb$
must be bounded by one in magnitude,
\begin{align*}
\max_{1\leq i\leq n, 1\leq j\leq d}{|\Xb_{ij}|} \leq
\nicefrac{1}{q}\leq 1.
\end{align*}

The following example supplies intuition for Theorem~\ref{thm:inhomogeneous_anticoncentration} and previews its proof in Section~\ref{sec_proof}.
It illustrates that for sufficiently tall and skinny matrices, the lower 
bound for the smallest singular value 
is meaningful and holds with a success probability
close to one.

\begin{example}
Consider 
Theorem~\ref{thm:inhomogeneous_anticoncentration} in the special
case when all elements of~$\Xb$
have the same variance. We show that this variance 
equals $1/n$, thus
decreases with increasing row dimension. 
Since $\gamma=\nicefrac{d}{n}$,  the lower bound
holds with probability close to 1 for sufficiently tall and skinny matrices.

Suppose that in assumption (\ref{ass2})  
\begin{align*}
\EE|\Xb_{ij}|^2 =\nicefrac{\kappa}{n}, \qquad 
1\leq i\leq n, \quad 1\leq j\leq d
\end{align*}
for some $\kappa\geq 1$.
Inserting this into the column variances (\ref{ass3}) gives 
\begin{align*}
\kappa=\max_{1\leq j\leq d} \sum_{i=1}^n{\EE|\Xb_{ij}|^2}\leq 1.
\end{align*}
This, together with $\kappa\geq 1$ implies 
$\kappa =1$, so that all elements have variance
\begin{align*}
\EE|\Xb_{ij}|^2 = 1/n, \qquad 1\leq i\leq n, \quad 1\leq  j\leq d.
\end{align*}
Thus, the variance decreases with increasing row dimension number.
Inserting $\EE|\Xb_{ij}|^2 = 1/n$ into the 
row variances (\ref{ass4}) gives
\begin{align*}
\frac{d}{n}=\max_{1\leq i\leq n}{\sum_{j=1}^d{\EE|\Xb_{ij}|^2}} \leq \gamma.
\end{align*}
This together with the assumption $n\geq d$ implies
$\gamma=\nicefrac{d}{n}\leq 1$, as required.
Inserting $\EE|\Xb_{ij}|^2 = 1/n$ into the 
minimal column  variances (\ref{ass5}) gives
\begin{align*}
1=\max_{1\leq j\leq d}{\sum_{i=1}^d{\EE|\Xb_{ij}|^2}}= \rhom
\end{align*}
so that
$1=\rhom\geq \sqrt{\gamma} = \sqrt{\nicefrac{d}{n}}$ in (\ref{ass5}) is automatically fulfilled.
The lower bound for $\sigma_d(\Xb)$ becomes
\begin{align*}
\sigma_d(\Xb) \geq  1 - \sqrt{\nicefrac{d}{n}}
-Cq^{-1/3}(\log{n})^{2/3}.
\end{align*}
Thus, 
the smallest singular value approaches 1 
with increasing probability as $\Xb$ becomes taller and skinnier.
    \end{example}
    
An open problem for future work is a better understanding for which assumptions (\ref{ass1})--(\ref{ass6}) are truly necessary for Theorem~\ref{thm:inhomogeneous_anticoncentration} to hold. We conjecture that many can be relaxed or even eliminated. Indeed, any improvement of Theorem~\ref{thm:inhomogeneous_anticoncentration} could result in significant improvements for and generalizations of our Theorem~\ref{lemma:general_rounding}.

\subsection{Proof of Theorem~\ref{lemma:general_rounding}}\label{sec_proof}
First we present a brief outline of the proof, and then
the proof proper.

\subsubsection*{Outline of proof}
The proof consists of four main steps.
\begin{enumerate}
\item We introduce the orthogonal projector $\Pb_{\Ab}$ 
onto the column space of $\Ab$. This allows us to focus on $\Pb_{\Ab}\Eb$. 
%
%
Weyl's inequality yields a lower bound on the smallest singular value of $(\Ib - \Pb_{\Ab})\Eb$ by lower bounding the smallest singular value of $\Eb$ and upper bounding the largest singular value of  $\Pb_\Ab\Eb$.
\item Application of
Theorem~\ref{thm:inhomogeneous_anticoncentration}
shows that the smallest singular value of $\Eb$ is sufficiently large. 
\item  
The largest singular value of the projection $\Pb_{\Ab}\Eb$
is small, because $\Pb_{\Ab}$ projects $\Eb$
on the low-dimensional subspace of dimension $d$, and
application of Hoeffding's inequality shows
that~$\Eb$ does not concentrate well in any low-dimensional subspace.
\item At last, we combine the bounds for the smallest singular value of $\Eb$ and the largest singular value of 
$\Pb_{\Ab}\Eb$ via a union bound on their probabilities.
\end{enumerate}


\subsubsection*{Proof}
We follow the four steps outlined above.
\begin{enumerate}
\item We decompose the task of lower bounding $\sigma_d(\Ab + \Eb)$ into two parts with the help of
an orthogonal projector $\Pb\in\R^{n\times n}$.
The effect of the orthogonal projector can be quantified by the singular value product inequalities
\cite[Theorem 3.3.16]{Horn2012}, which together with 
$\|\Pb\|_2=1$ imply
\begin{align}\label{eqn:pd566}
\sigma_{d}(\Pb(\Ab+\Eb)) \leq \sigma_{d}(\Ab+\Eb)\|\Pb\|_2 \leq \sigma_{d}(\Ab+\Eb). 
\end{align}
Let $\Pb_{\Ab} \in \mathbb{R}^{n \times n}$ be the orthogonal projector on the column space of $\Ab$ and 
$\Pbb \equiv\Ib_n-\Pb_{\Ab}$ be the orthogonal projector onto the left nullspace of $\Ab$. From~(\ref{eqn:pd566}) follows
\begin{align}
    \sigma_d(\Ab + \Eb) &\geq \sigma_d(\Pbb(\Ab + \Eb))
    = \sigma_d(\Pbb\Ab + \Pbb\Eb)\nonumber\\
  &  =\sigma_d(\Pbb\Eb).\label{eqn:pd1a}
\end{align}
Weyl's inequality \cite[Theorem III.2.1]{bhatia2013matrix}
implies
\begin{align}\label{eqn:pd1b}
    \sigma_d(\Pbb\Eb) = \sigma_d(\Eb - \Pb_{\Ab}\Eb) \geq \sigma_d(\Eb) - \|\Pb_{\Ab}\Eb\|_2.
\end{align}
We have now broken our task into two parts. 
First, we must make sure that the smallest singular value of the random matrix $\Eb$ is large enough. 
Second, the projection of $\Eb$ onto the $d$-dimensional 
column space of $\Ab$ must be small, that is, $\Eb$ should not concentrate in any $d$-dimensional space.

\item To bound $\sigma_d(\Eb)$ from below, set $\Xb = \frac{1}{\Rcal \cdot \sqrt{n}} \Eb$, where $\EE[\Xb] = \zero$. We
show that
$\Xb$ satisfies the assumptions of Theorem \ref{thm:inhomogeneous_anticoncentration}
by setting $\kappa = 1$, $\gamma = \nicefrac{d}{n}$, 
$\rhom = \nu$, 
and $q = n^{\nicefrac{1}{10}}\cdot \gamma^{-\nicefrac{1}{18}}$,

From $\max_{i,j}|\Eb_{ij}| \leq \Rcal$ and 
$q = n^{\nicefrac{1}{10}}\cdot \gamma^{-\nicefrac{1}{18}}$
follows (\ref{ass1}),
\begin{align*}
\max_{1\leq i\leq n, 1\leq j\leq d} |\Xb_{ij}| \leq \frac{1}{\sqrt{n}} \leq \frac{1}{q}.
\end{align*}
This implies (\ref{ass2}),
\begin{align*}
\max_{1\leq i\leq n, 1\leq j\leq d} \EE|\Xb_{ij}|^2 \leq \frac{1}{n} = \frac{\kappa}{n},
\end{align*}
which in turn implies (\ref{ass3}),
\begin{align*}
\max_{1\leq j\leq d} \sum_{i=1}^n \EE|\Xb_{ij}|^2 \leq \sum_{i=1}^n \frac{1}{n} = 1,
\end{align*}
as well as (\ref{ass4}),
\begin{align*}
\max_{1\leq i\leq n} \sum_{j=1}^d \EE|\Xb_{ij}|^2 \leq \sum_{j=1}^d \frac{1}{n} = \frac{d}{n}=\gamma.
\end{align*}
From
\begin{align*}
\EE[\Xb_{ij}^2] = \EE[\Xb_{ij}]^2 + \Var[\frac{1}{\Rcal \cdot \sqrt{n}} \Eb_{ij}] = \frac{1}{n\cdot \Rcal^2} \Var[\Eb_{ij}]
\end{align*}
and (\ref{eqn:nupd}) follows (\ref{ass5})
\begin{align*}
\min_{1\leq j\leq d} \sum_{i=1}^n \EE|\Xb_{ij}|^2 = \min_{1\leq j\leq d} \sum_{i=1}^n \frac{1}{\Rcal^2 \cdot n}\cdot\Var[\Eb_{ij}] = \nu = \rhom.
\end{align*}
We now focus on the assumption (\ref{ass6}), $\sqrt{\log n} \leq q \leq n^{\frac{1}{10}}\kappa^{-\frac{1}{9}}\gamma^{-\frac{1}{18}}$. The second inequality is immediately satisfied, because $\kappa=1$ and $q = n^{\nicefrac{1}{10}}\cdot \gamma^{-\nicefrac{1}{18}}$. The first inequality is equivalent to
\begin{align*}
    \log n \leq n^{\nicefrac{14}{45}}\cdot d^{-\nicefrac{1}{9}}, 
\end{align*}
hence
\begin{align}\label{eqn:pd23}
d \leq \frac{n^{\nicefrac{14}{5}}}{(\log{n})^9}.
\end{align}

%
%
Let the assumption $n \geq 836$ hold, and suppose 
$d^4 \geq n$. Then $2d^2\sqrt{\nicefrac{\log n}{n}} > 2$
and Theorem~\ref{lemma:general_rounding} vacuously holds, since the right-hand side of the bound is negative. 

Now suppose that $d^4 < n$. Then for $n \geq 836$ 
(\ref{eqn:pd23}) holds,
\begin{align*}
d< n^{\nicefrac{1}{4}} \leq \nicefrac{n^{\nicefrac{14}{5}}}{(\log{n})^9}. 
\end{align*}
Thus,  $\Xb$ satisfies the assumptions of Theorem~\ref{thm:inhomogeneous_anticoncentration},
so that with probability at least $1 - n^{-c}$,
%
%
\begin{align*}
    \sigma_d(\Xb) \geq \sqrt{\nu} - \sqrt{\frac{d}{n}} - \frac{C(\log{n})^{2/3}}{n^{1/30}} \cdot \left(\frac{d}{n}\right)^{\nicefrac{1}{54}},
\end{align*}
where $c$ and $C$ are absolute constants.
Recall the definition of $\Xb$ and multiply both sides by $\Rcal \cdot \sqrt{n}$,
\begin{align}\label{eqn:pd1c}
    \sigma_d(\Eb) \geq \Rcal\sqrt{\nu \cdot n} - \Rcal\sqrt{d} - \frac{C\cdot \Rcal\cdot(\log{n})^{2/3}}{n^{1/30}} \cdot \left(\frac{d}{n}\right)^{\nicefrac{1}{54}} \cdot \sqrt{n}.
\end{align}
\item To bound $\|\Pb_{\Ab}\Eb\|_2$ from above,
%
first consider the case where $\Pb_{\Ab}$ projects onto a one-dimensional space, so that $\Pb_{\Ab} =\ub\ub^T$ for some unit vector $\ub\in  \mathbb{R}^n$, and
\begin{align*}
    \|\Pb_{\Ab}\Eb\|_2 = \|\ub\ub^T\Eb\|_2 = \|\Eb^T\ub\|_2 .
\end{align*}


We now bound each entry of $\Eb^T\ub$ via
Theorem~\ref{thm:hoeffding}. First check the assumptions. From
$\max_{i\textcolor{orange}{,}j}{|\Eb_{ij}|}\leq \Rcal$
follows  $|\Eb_{ij}\ub_i| \leq |\ub_i|\cdot\Rcal$
for a fixed unit vector $\ub \in \R^n$.

Hence, the values of the random variable $\Eb_{ij}\ub_i$ are in the interval $[m_{ij}, M_{ij}]$, with $(M_{ij} - m_{ij})^2 \leq 4\ub_i^2 \Rcal^2$. From $\|\ub\|_2= 1$ follows
\begin{align*}
\sum_{i=1}^n (M_{ij} - m_{ij})^2 \leq 4\Rcal^2.
\end{align*}
Applying Theorem~\ref{thm:hoeffding} gives 
\begin{align*}
    \PP\left[\Big|\sum_{i=1}^n \Eb_{ij} \ub_i\Big| \geq t\right]
    &\leq 
    2\exp\left(\frac{-2t^2}{\sum_{i=1}^n (M_{ij} - m_{ij})^2}\right), \qquad 1\leq j\leq d\\
    &\leq 2\exp\left(\frac{-t^2}{2\Rcal^2}\right).
\end{align*}
The vector norm relations 
\begin{align*}
\|\Eb^T\ub\|_2 \leq \|\Eb^T\ub\|_1 = \sum_{j=1}^d \Big|\sum_{i=1}^n \Eb_{ij} \ub_i\Big|,
\end{align*}
%
%
followed by application of the above inequality with 
$t=2\Rcal\sqrt{\log n}$ and application of the union bound from Section~\ref{sec_union} over all $j=1\ldots d$ gives
\begin{align}\label{eqn:pd1110}
    \PP\left(\|\Eb^T\ub\|_2 \geq 2d\cdot\Rcal \sqrt{\log n}\right) 
    \leq d \cdot  2\exp\left(\frac{-4 \cdot \Rcal^2 \log n}{2\Rcal^2}\right) = \frac{2d}{n^2}.
\end{align}

We extend this to the case where $\Pb_{\Ab}$ projects on
a $d$-dimensional subspace, so that $\Pb_{\Ab} = \sum_{j=1}^d \ub_j\ub_j^T$ for orthonormal vectors
$\ub_j \in \mathbb{R}^n$. The triangle inequality implies
\begin{align*}
\left\|\Pb_{\Ab}\Eb\right\|_2  = \left\|\sum_{j=1}^d \ub_j\ub_j^T\Eb\right\|_2\leq \sum_{j=1}^d \|\ub_j\ub_j^T\Eb\|_2.
\end{align*}
Application of (\ref{eqn:pd1110}) and application of the union bound from
Section~\ref{sec_union} over all $d$ vectors $\ub_j$ gives
\begin{align}\label{eqn:pd1d}
    \PP\left[\|\Pb_{\Ab}\Eb\|_2 \geq 2d^2\cdot\Rcal\sqrt{\log n}\right]
    \leq \frac{2d^2}{n^2}
\end{align}
\item Finally, use $\Abtil = \Ab + \Eb$ and combine (\ref{eqn:pd1a}),~(\ref{eqn:pd1b}),~(\ref{eqn:pd1c}), and~(\ref{eqn:pd1d}) to get
\begin{align*}
    \sigma_d(\Abtil) &\geq \Rcal\sqrt{\nu \cdot n} - \Rcal\sqrt{d} - \frac{C\cdot \Rcal\cdot\log^{2/3}(n)}{n^{1/30}} \cdot \left(\frac{d}{n}\right)^{\nicefrac{1}{54}} \cdot \sqrt{n}\\
    &\qquad\qquad- 2d^2\cdot\Rcal\sqrt{\log n}.
\end{align*}
Application of a union bound shows that
the above holds with probability at least $1-\nicefrac{1}{n^{c}}-\nicefrac{2d^2}{n^2}$. At last, factor out $\Rcal \sqrt{n}$ and use the abbreviation $\varepsilon_{n,d}$.
\end{enumerate}

\subsection{Theorem~\ref{lemma:general_rounding} is tight}\label{s_lower}

To show that Theorem~\ref{lemma:general_rounding} is asymptotically tight, we exhibit $n \times d$ matrices $\Ab$ whose stochastically rounded version $\Abtil$ has a smallest singular value no larger than $\Rcal\sqrt{n \nu}$ times a small constant, slightly larger than one, with $\nu$ from~(\ref{eqn:nupd}).

\begin{theorem}
    For every $n \geq d$ and $1\leq \ell \leq n$, there exists a matrix $\Ab \in \R^{n \times d}$ with normalized minimum column variance $\nu = \nicefrac{\ell}{n}$, whose stochastically rounded version $\Abtil$ has a smallest singular value
    %
    \begin{align*}
    \sigma_d(\Abtil) \leq  \left(1+\sqrt{\nicefrac{1}{(d-1)}}\right) \, \Rcal\sqrt{\nu \, n}.
    \end{align*}
\end{theorem}
\begin{proof}
    Let $\Fcal = \{0,1\}$ and set $\Abtil = \Ab + \Eb$. Fix $\nu \in (0,1)$ such that $\nu n = \ell$ for some integer $1\leq \ell \leq d$. Then, $\nu = \nicefrac{\ell}{n}$. Let the entries in the first $\ell$ rows of $\Ab$ be equal to $\nicefrac{1}{2}$ and let the remaining $n-\ell$ rows be zero. 

    Then, $\Var(\Eb_{ij}) = \nicefrac{1}{4}$ for $1\leq i \leq \ell$, $1\leq j\leq d$; $\Var(\Eb_{ij}) = 0$ for
    $\ell+1\leq i\leq n$, $1\leq j\leq d$;
    $\Rcal = \max_{i,j}|\Eb_{ij}| = \nicefrac{1}{2}$; and 
    \begin{align*}
    \frac{1}{n\Rcal^2}\sum_{i=1}^n \Var(\Eb_{ij}) = \frac{4}{n}\sum_{i=1}^{\ell} \frac{1}{4} = \frac{\ell}{n} = \nu,
    \qquad 1\leq j\leq d.
    \end{align*}
    %
    %
    By construction, $\Ab$ has a single non-zero singular value, and $\sigma_j(\Ab) = 0$,  $2\leq j\leq d$. 
    Weyl's inequality \cite[Theorem III.2.1]{bhatia2013matrix}
    implies
    \begin{align}\label{eqn:pd1}
        \sigma_d(\Abtil) = \sigma_d(\Ab + \Eb)
        \leq \sigma_2(\Ab) + \sigma_{d-1}(\Eb) = \sigma_{d-1}(\Eb).
    \end{align}
    Since $\Eb_{ij}=\pm\nicefrac{1}{2}$ for  $1\leq i\leq \ell$, $1\leq j\leq d$ and zero otherwise,
    we have $\|\Eb\|_F^2 = \ell d/4$. 
    Inserting this into the Frobenius norm
    \begin{align*}
        \|\Eb\|_F^2 = \sum_{i=1}^{d} \sigma_i^2(\Eb) \geq \sum_{i=1}^{d-1} \sigma_i^2(\Eb) \geq (d-1) \, \sigma_{d-1}^2(\Eb),
    \end{align*}
    gives
    \begin{align*}
        \sigma_{d-1}^2(\Eb) \leq \frac{1}{d-1} \|\Eb\|_F^2 = \frac{d}{d-1}\cdot \frac{1}{4}\ell = \frac{d}{4(d-1)} n\, \nu.
    \end{align*}
    
    At last, take square roots, combine with~(\ref{eqn:pd1}), and abbreviate $\Rcal=\nicefrac{1}{2}$.
\end{proof}
\section{Experiments} \label{exps}

The numerical experiments illustrate the behavior of the smallest singular value under SR-nearness rounding,
 and support our results from Section~\ref{s_bounds}. 
The experiments are designed so that the effects 
of fine-grained changes
in precision can be easily discerned. The scripts
for reproducing the numerical experiment are available in our git repository\footnote{\url{https://github.com/cboutsikas/stoch_rounding_iplicit_reg}}.

\paragraph{Design of experiments}
We investigate the dependence of $\sigma_d(\Abtil)$ on various factors, such as the aspect ratio $\nicefrac{n}{d}$, the smallest singular value $\sigma_{d}(\Ab)$, and the minimal column variance $\nu$. To provide statistical significance, for each $\Ab$,
we generate 100 stochastically rounded matrices~$\Abtil$ and
compute their singular values $\sigma_d(\Abtil)$.

Furthermore, we examine both fixed-point arithmetic with base $\beta = 10$ and floating-point arithmetic, where we demote the matrix in lower precision using SR-nearness.
Experiments in fixed-point arithmetic are presented for
rank-deficient matrices (Section~\ref{exp_fixed_rank_def}) and full-rank matrices (Section~\ref{exp_fixed_rank_full}); and in floating point arithmetic for rank-deficient matrices (Section~\ref{exp_fl_rank_def}) and matrices with controlled $\nu$ (Section~\ref{exp_fl_nu}).

\paragraph{Matrices}
The matrices have $n = 10^4$ rows and $d = 10$, $100$, and $1000$ columns.
Their elements are drawn from different distributions, including skewed and non-symetric ones.
We start with matrices $\Ab\in\R^{n\times d}$ whose elements are independent identically distributed random variables from the standard normal distribution $\mathcal{N}(0,1)$ or the \texttt{Lognormal}$(0,3)$ distribution. Subsequently, we adjust the smallest singular value to fit the desired setting, e.g. for singular~$\Ab$ we force $\sigma_{d}(\Ab) = 0$. 
In the experiments where we control $\nu$, we modify the elements of $\Ab$ directly and force singularity by setting two columns equal to each other.

We present our main experimental findings in Tables~\ref{table_fig_5.1}-\ref{table_fig_5.6} and provide additional plots in Appendix~\ref{append_plots}.

\subsection{Experiments with fixed point arithmetic} \label{exp_fixed}
SR nearness rounds the elements of $\Ab$ to elements of the set $\mathcal{F}^{\{p\}}$, $1\leq p\leq 3$,
where
\begin{align}\label{eqn:eqFp}
\mathcal{F}^{\{p\}} = \{\pm\,\nicefrac{m}{10 ^{p}}, \quad \text{for all integers}\ m = \underbrace{0,1,2,\ldots,10^{p}-1}_{\leq \,p\ \text{digits}}\}\  \cup\  \{\pm 1\}.    
\end{align}
This is equivalent to rounding to a signed base-$10$ fixed-point precision with at most $p$ digits in the fractional part.

\subsubsection{Tables~\ref{table_fig_5.1},~\ref{table_fig_5.2} and Figures~\ref{exp_fixed_rank_def_plot},~\ref{exp_fixed_rank_def_lognorm_plot}: Rank deficient matrices} \label{exp_fixed_rank_def}
The rank deficient matrices $\Ab$ have a smallest singular value 
$\sigma_d(\Ab) = 0$. The goal is to understand how 
the aspect ratio $\nicefrac{n}{d}$ 
affects the behavior of $\sigma_{d}(\tilde{\Ab})$ on matrices drawn from different distributions.

\begin{table}[h]
    \centering
    \renewcommand{\arraystretch}{1.5}
    \caption{The percentage of matrices (recall that we perform 100 stochastic roundings for each parameter setting, see also~\ref{exp_fixed_rank_def_plot}) violating the estimate of eqn.~(\ref{eq:maineqapprox}). The elements of $\Ab$ are in $\mathcal{N}(0,1)$, while the elements $\tilde{\Ab}_{ij}$ belong to $\mathcal{F}^{\{p\}}$, thus $\nceil{\Ab_{ij}} - \nfloor{\Ab_{ij}} \leq 10^{-p}$, for $p=1,\ldots,3$. Notice that when the number of columns is $d=10$ or $100$, the bound $\mathcal{R} \sqrt{n \nu}$ 
    is violated in approximately $25\%$ to $50\%$ of the test cases; the bound is always violated for $d=1,000$ (squarish matrix). However, a very mild relaxation of the bound from 
    $\mathcal{R} \sqrt{n \nu}$ to $0.9\cdot \mathcal{R} \sqrt{n \nu}$ immediately fixes this issue, resulting in zero violations in all settings. Additionally, the relative error between the estimate provided by our bound and the \textit{minimum} observed increase in the smallest singular value of the rounded matrix remains consistently below $6\%$. More precisely, we compute $s_{\min} = \min\{\sigma_{\min}(\Abtil)\}$ over all 100 roundings for a specific parameter setting and report the relative error  
    $1- \nicefrac{s_{\min}}{\mathcal{R} \sqrt{n \nu}}$. In light of the small relative error, it is clear that our bound provides a useful estimate for the magnitude of regularization, even for \textit{very modest} aspect ratios (see Remark~\ref{remarks_exps}).}
    \label{table_fig_5.1}
    \begin{tabular}{|c|cccccc|ccc|}
    \hline
   \rule{0pt}{3ex} $\mathbf{d}$ & \multicolumn{6}{c|}{$ \% \left( \sigma_{\min}(\Abtil) < c \cdot\mathcal{R} \sqrt{n \nu} \right) $} &  \multicolumn{3}{c|} {$1- \nicefrac{s_{\min}}{\mathcal{R} \sqrt{n \nu}}$} \\
    \hline
        \hline
          & \multicolumn{2}{c}{$\mathbf{p=1}$}  &  \multicolumn{2}{c}{$\mathbf{p=2}$} & \multicolumn{2}{c|}{$\mathbf{p=3}$}  & $\mathbf{p=1}$ &  $\mathbf{p=2}$ & $\mathbf{p=3}$ \\
          & $c=1$ & $c=.9$ & $c=1$ & $c=.9$ & $c=1$ & $c=.9$ & & & \\
        \hline
        $\mathbf{10}$ & $26 \%$ & $0\%$ & $46 \%$ & $0\%$ & $30 \%$ & 0\% & .01 & .01 & .01  \\
        \hline
        $\mathbf{100}$ & $48 \%$ & 0\% & $37 \%$ & 0\% &  $51 \%$ &  $0 \%$ & .02 & .01 & .02 \\
        \hline
        $\mathbf{1000}$ & $100 \%$ & 0\% &  $100 \%$ & 0\% & $100 \%$ & 0\% & .06 & .06 & .06\\
        \hline
    \end{tabular}
\end{table}

\begin{table}[h]
    \centering
    \renewcommand{\arraystretch}{1.5}
    \caption{The percentage of matrices (recall that we perform 100 stochastic roundings for each parameter setting, see also~\ref{exp_fixed_rank_def_lognorm_plot}) violating the estimate of eqn.~(\ref{eq:maineqapprox}).  
    The elements of $\Ab$ are in $\texttt{Lognormal}(0,3)$, while the elements $\tilde{\Ab}_{ij}$ belong to $\mathcal{F}^{\{p\}}$, thus $\nceil{\Ab_{ij}} - \nfloor{\Ab_{ij}} \leq 10^{-p}$, for $p=1,\ldots,3$. Notice that when the number of columns is $d=10$ or $100$, the bound $\mathcal{R} \sqrt{n \nu}$ 
    is violated in approximately $0\%$ to $35\%$ of the test cases; the bound is always violated for $d=1,000$ (squarish matrix). However, a very mild relaxation of the bound from 
    $\mathcal{R} \sqrt{n \nu}$ to $0.9\cdot \mathcal{R} \sqrt{n \nu}$ immediately fixes this issue, resulting in zero violations in all settings. Additionally, the relative error between the estimate provided by our bound and the \textit{minimum} observed increase in the smallest singular value of the rounded matrix remains consistently below $6\%$. More precisely, we compute $s_{\min} = \min\{\sigma_{\min}(\Abtil)\}$ over all 100 roundings for a specific parameter setting and report the relative error  $1- \nicefrac{s_{\min}}{\mathcal{R} \sqrt{n \nu}}$. If $s_{\min} > \mathcal{R} \sqrt{n \nu}$, we mark the respective entry as N/A. 
    } 
    \label{table_fig_5.2}
    \begin{tabular}{|c|cccccc|ccc|}
    \hline
   \rule{0pt}{3ex} $\mathbf{d}$ & \multicolumn{6}{c|}{$ \% \left( \sigma_{\min}(\Abtil) < c \cdot\mathcal{R} \sqrt{n \nu} \right) $} &  \multicolumn{3}{c|} {$1- \nicefrac{s_{\min}}{\mathcal{R} \sqrt{n \nu}}$} \\
    \hline
        \hline
          & \multicolumn{2}{c}{$\mathbf{p=1}$}  &  \multicolumn{2}{c}{$\mathbf{p=2}$} & \multicolumn{2}{c|}{$\mathbf{p=3}$}  & $\mathbf{p=1}$ &  $\mathbf{p=2}$ & $\mathbf{p=3}$ \\
          & $c=1$ & $c=.9$ & $c=1$ & $c=.9$ & $c=1$ & $c=.9$ & & & \\
        \hline
        $\mathbf{10}$ & $0 \%$ & $0\%$ & $36 \%$ & $0\%$ & $22 \%$ & 0\% & N/A & .01 & .01  \\
        \hline
        $\mathbf{100}$ & $0 \%$ & 0\% & $15 \%$ & 0\% &  $34 \%$ &  $0 \%$ & N/A & .01 & .01 \\
        \hline
        $\mathbf{1000}$ & $100 \%$ & 0\% &  $100 \%$ & 0\% & $100 \%$ & 0\% & .04 & .05 & .06\\
        \hline
    \end{tabular}
\end{table}

\begin{remark}\label{remarks_exps}
We again emphasize that our bound predicts the behavior of the smallest singular value of the stochastically-rounded matrix \textit{asymptotically}. Modifying our bound by just a small constant to be $0.9\cdot \mathcal{R} \sqrt{n \nu}$ instead of $\mathcal{R} \sqrt{n \nu}$, results in zero violations. In this case, the relative error on the right-hand side of the table is not relevant, since even the smallest observed singular value $s_{\min}$ of the rounded matrices would exceed the estimate $0.9\cdot \mathcal{R} \sqrt{n \nu}$. This highlights that our estimate is very accurate as a lower bound for the regularization even for almost square matrices. 
\end{remark}

\subsubsection{Tables~\ref{table_fig_5.3},~\ref{table_fig_5.4} and Figures~\ref{exp_fixed_rank_full_plot},~\ref{exp_fixed_rank_full_lognorm_plot}: Full rank} \label{exp_fixed_rank_full}
The full-rank matrices $\Ab$ have the smallest singular value 
$\sigma_d(\Ab) = 10^{-2}$. The goal is to understand how the aspect ratio $\nicefrac{n}{d}$ affects the behavior of $\sigma_{d}(\tilde{\Ab})$.

\begin{table}[h]
    \centering
    \renewcommand{\arraystretch}{1.5}
    \caption{The percentage of matrices (recall that we perform 100 stochastic roundings for each parameter setting, see also~\ref{exp_fixed_rank_full_plot}) violating the estimate of eqn.~(\ref{eq:maineqapprox}).  
    The elements of $\Ab$ are in $\mathcal{N}(0,1)$, while the elements $\tilde{\Ab}_{ij}$ belong to $\mathcal{F}^{\{p\}}$, thus $\nceil{\Ab_{ij}} - \nfloor{\Ab_{ij}} \leq 10^{-p}$, for $p=1,\ldots, 3$. Notice that when the number of columns is $d=10$ or $100$, the bound $\mathcal{R} \sqrt{n \nu}$ 
    is violated in approximately $0\%$ to $40\%$ of the test cases; the bound is almost always violated for $d=1,000$ (squarish matrix). However, a very mild relaxation of the bound from 
    $\mathcal{R} \sqrt{n \nu}$ to $0.9\cdot \mathcal{R} \sqrt{n \nu}$ immediately fixes this issue, resulting in zero violations in all settings. Additionally, the relative error between the estimate provided by our bound and the \textit{minimum} observed increase in the smallest singular value of the rounded matrix remains consistently below $6\%$. More precisely, we compute $s_{\min} = \min\{\sigma_{\min}(\Abtil)\}$ over all 100 roundings for a specific parameter setting and report the relative error  
    $1- \nicefrac{s_{\min}}{\mathcal{R} \sqrt{n \nu}}$. If $s_{\min} > \mathcal{R} \sqrt{n \nu}$, we mark the respective entry as N/A. 
    } 
    \label{table_fig_5.3}
    \begin{tabular}{|c|cccccc|ccc|}
    \hline
   \rule{0pt}{3ex} $\mathbf{d}$ & \multicolumn{6}{c|}{$ \% \left( \sigma_{\min}(\Abtil) < c \cdot\mathcal{R} \sqrt{n \nu} \right) $} &  \multicolumn{3}{c|} {$1- \nicefrac{s_{\min}}{\mathcal{R} \sqrt{n \nu}}$} \\
    \hline
        \hline
          & \multicolumn{2}{c}{$\mathbf{p=1}$}  &  \multicolumn{2}{c}{$\mathbf{p=2}$} & \multicolumn{2}{c|}{$\mathbf{p=3}$}  & $\mathbf{p=1}$ &  $\mathbf{p=2}$ & $\mathbf{p=3}$ \\
          & $c=1$ & $c=.9$ & $c=1$ & $c=.9$ & $c=1$ & $c=.9$ & & & \\
        \hline
        $\mathbf{10}$ & $37 \%$ & $0\%$ & $29 \%$ & $0\%$ & $0 \%$ & 0\% & .02 & .02 & N/A  \\
        \hline
        $\mathbf{100}$ & $39 \%$ & 0\% & $36 \%$ & 0\% &  $0 \%$ &  $0 \%$ & .01 & .02 & N/A \\
        \hline
        $\mathbf{1000}$ & $100 \%$ & 0\% &  $100 \%$ & 0\% & $96 \%$ & 0\% & .06 & .06 & .03\\
        \hline
    \end{tabular}
\end{table}

\begin{table}[h]
    \centering
    \renewcommand{\arraystretch}{1.5}
    \caption{The percentage of matrices (recall that we perform 100 stochastic roundings for each parameter setting, see also~\ref{exp_fixed_rank_full_lognorm_plot}) violating the estimate of eqn.~(\ref{eq:maineqapprox}).  
    The elements of $\Ab$ are in $\texttt{Lognormal}(0,3)$, while the elements $\tilde{\Ab}_{ij}$ belong to $\mathcal{F}^{\{p\}}$, thus $\nceil{\Ab_{ij}} - \nfloor{\Ab_{ij}} \leq 10^{-p}$, for $p=1,\ldots, 3$. Notice that when the number of columns is $d=10$ or $100$, the bound $\mathcal{R} \sqrt{n \nu}$ 
    is violated in approximately $0\%$ to $30\%$ of the test cases; the bound is almost always violated for $d=1,000$ (squarish matrix). However, a very mild relaxation of the bound from 
    $\mathcal{R} \sqrt{n \nu}$ to $0.9\cdot \mathcal{R} \sqrt{n \nu}$ immediately fixes this issue, resulting in zero violations in all settings. Additionally, the relative error between the estimate provided by our bound and the \textit{minimum} observed increase in the smallest singular value of the rounded matrix remains consistently below $6\%$. More precisely, we compute $s_{\min} = \min\{\sigma_{\min}(\Abtil)\}$ over all 100 roundings for a specific parameter setting and report the relative error  
    $1- \nicefrac{s_{\min}}{\mathcal{R} \sqrt{n \nu}}$. If $s_{\min} > \mathcal{R} \sqrt{n \nu}$, we mark the respective entry as N/A.} 
    \label{table_fig_5.4}
    \begin{tabular}{|c|cccccc|ccc|}
    \hline
   \rule{0pt}{3ex} $\mathbf{d}$ & \multicolumn{6}{c|}{$ \% \left( \sigma_{\min}(\Abtil) < c \cdot\mathcal{R} \sqrt{n \nu} \right) $} &  \multicolumn{3}{c|} {$1- \nicefrac{s_{\min}}{\mathcal{R} \sqrt{n \nu}}$} \\
    \hline
        \hline
          & \multicolumn{2}{c}{$\mathbf{p=1}$}  &  \multicolumn{2}{c}{$\mathbf{p=2}$} & \multicolumn{2}{c|}{$\mathbf{p=3}$}  & $\mathbf{p=1}$ &  $\mathbf{p=2}$ & $\mathbf{p=3}$ \\
          & $c=1$ & $c=.9$ & $c=1$ & $c=.9$ & $c=1$ & $c=.9$ & & & \\
        \hline
        $\mathbf{10}$ & $0 \%$ & $0\%$ & $8 \%$ & $0\%$ & $0 \%$ & 0\% & N/A & .01 & N/A  \\
        \hline
        $\mathbf{100}$ & $2 \%$ & 0\% & $27 \%$ & 0\% &  $0 \%$ &  $0 \%$ & .001 & .01 & N/A \\
        \hline
        $\mathbf{1000}$ & $100 \%$ & 0\% &  $100 \%$ & 0\% & $95 \%$ & 0\% & .05 & .06 & .03\\
        \hline
    \end{tabular}
\end{table}

\subsection{Experiments with floating point numbers} \label{exp_fl}
Given a matrix $\Ab$ in double precision, we form $\Abtil$ by demoting $\Ab$ to single precision using SR-nearness. To achieve this, we emulate the computations with the \texttt{chop} library \cite{higham2019simulating}. 

\subsubsection{Table~\ref{table_fig_5.5}, and Figure~\ref{exp_fl_rank_def_plot}: Rank deficient matrices} \label{exp_fl_rank_def}
The rank deficient matrices $\Ab$ have a smallest singular value 
$\sigma_d(\Ab) = 0$. The goal is to understand how the aspect ratio $\nicefrac{n}{d}$ affects the behavior of $\sigma_{d}(\tilde{\Ab})$ on matrices drawn from different distributions when $\Ab$ is being demoted to a lower precision via SR-nearness.

\begin{table}[h]
    \centering
    \renewcommand{\arraystretch}{1.5}
    \caption{The percentage of matrices (recall that we perform 100 stochastic roundings for each parameter setting, see also~\ref{exp_fl_rank_def_plot}) violating the estimate of eqn.~(\ref{eq:maineqapprox}). The elements of $\Ab$ are drawn either from $\mathcal{N}(0,1)$ or $\texttt{Lognormal}(0,3)$, while the elements of $\tilde{\Ab}_{ij}$ are obtained by stochastically rounding the corresponding elements $\Ab_{ij}$ to single precision.  We present results for both tested distributions. Notice that when the number of columns is $d=10$ or $100$, the bound $\mathcal{R} \sqrt{n \nu}$ 
    is violated in approximately $0\%$ to $30\%$ of the test cases; the bound is always violated for $d=1,000$ (squarish matrix) when elements drawn from $\mathcal{N}(0,1)$, whereas there are no violations when the elements are drawn from $\texttt{Lognormal}(0,3)$. Still, a very mild relaxation of the bound from 
    $\mathcal{R} \sqrt{n \nu}$ to $0.8\cdot \mathcal{R} \sqrt{n \nu}$ effectively eliminates almost every violation across all settings.
    Furthermore, we compute $s_{\min} = \min\{\sigma_{\min}(\Abtil)\}$ over all 100 roundings for a specific parameter setting and report the relative error  
    $1- \nicefrac{s_{\min}}{\mathcal{R} \sqrt{n \nu}}$. If $s_{\min} > \mathcal{R} \sqrt{n \nu}$, we mark the respective entry as N/A.}
    \label{table_fig_5.5}
    \begin{tabular}{|c|cccc|cc|}
    \hline
   \rule{0pt}{3ex} $\mathbf{d}$ & \multicolumn{4}{c|}{$ \% \left( \sigma_{\min}(\Abtil)) < \mathcal{R} \sqrt{n \nu} \right) $} &  \multicolumn{2}{c|} {$1- \nicefrac{s_{\min}}{\mathcal{R} \sqrt{n \nu}}$} \\
    \hline
        \hline
          & \multicolumn{2}{c}{$\mathbf{\mathcal{N}(0,1)}$}  &  \multicolumn{2}{c|}{$\mathbf{\texttt{Lognormal}(0,3)}$} & $\mathbf{\mathcal{N}(0,1)}$  & $\mathbf{\texttt{Lognormal}(0,3)}$  \\
    & $c=1$ & $c=.8$ & $c=1$ & $c=.8$ & & \\
        \hline
        $\mathbf{10}$ & $0 \%$ & $0 \%$ & $2 \%$ & $0 \%$  & N/A & $.07$  \\
        \hline
        $\mathbf{100}$ & $3 \%$ & $0 \%$ & $28 \%$ & $3 \%$  & .001 & $.2$  \\
        \hline
        $\mathbf{1,000}$ & $100 \%$ & $0 \%$ &  $0 \%$ & $0 \%$ & .04 & N/A \\
        \hline
    \end{tabular}
\end{table}


\subsubsection{Table~\ref{table_fig_5.6}, and Figure~\ref{exp_fl_nu_plot}: Matrices with controlled $\nu$} \label{exp_fl_nu}
We start with a $10^{4} \times d$ matrix whose elements are independent identically distributed random variables $\mathcal{N}(0,1)$, and enforce rank deficiency by setting two columns equal to each other. 
For each column dimension $d$, we create two matrices: $\Ab^{h}$ with a 'high' value of $\nu$ and $\Ab^{l}$ with a 'low' value of $\nu$ such as $\nicefrac{\nu(\Ab^{h})}{\nu(\Ab^{l})} \approx 100$.
The goal is to also understand how $\nu$ might affect the behavior of $\sigma_{d}(\Abtil)$.

\begin{table}[h]
    \centering
    \renewcommand{\arraystretch}{1.5}
    \caption{The percentage of matrices (recall that we perform 100 stochastic roundings for each parameter setting, see also~\ref{exp_fl_nu_plot}) violating the estimate of eqn.~(\ref{eq:maineqapprox}). The elements of $\Ab$ are drawn from $\mathcal{N}(0,1)$ and subsequently we construct two matrices $\Ab^{h}, \, \Ab^{l}$ corresponding to a 'high' and 'low' value of $\nu$ respectively, meaning that $\nicefrac{\nu(\Ab^{h})}{\nu(\Ab^{l})} \approx 100$. Notice that when the number of columns is $d=10$ or $100$, the bound $\mathcal{R} \sqrt{n \nu}$ 
    is violated in approximately $0\%$ to $55\%$ of the test cases; the bound is almost always violated for $d=1,000$ (squarish matrix). As indicated by our theory, matrices with higher value of $\nu$ lead to less violations (excluding the ill-advised case of $d=1,000$). Again, a very mild relaxation of the bound from 
    $\mathcal{R} \sqrt{n \nu}$ to $0.8\cdot \mathcal{R} \sqrt{n \nu}$ effectively eliminates almost every violation across all settings.
    Furthermore, we compute $s_{\min} = \min\{\sigma_{\min}(\Abtil)\}$ over all 100 roundings for a specific parameter setting and report the relative error  
    $1- \nicefrac{s_{\min}}{\mathcal{R} \sqrt{n \nu}}$. If $s_{\min} > \mathcal{R} \sqrt{n \nu}$, we mark the respective entry as N/A.}
    \label{table_fig_5.6}
    \begin{tabular}{|c|cccc|cc|}
    \hline
   \rule{0pt}{3ex} $\mathbf{d}$ & \multicolumn{4}{c|}{$ \% \left( \sigma_{\min}(\Abtil)) < \mathcal{R} \sqrt{n \nu} \right) $} &  \multicolumn{2}{c|} {$1- \nicefrac{s_{\min}}{\mathcal{R} \sqrt{n \nu}}$} \\
    \hline
        \hline
          & \multicolumn{2}{c}{$\mathbf{\Abtil^{h}}$}  &  \multicolumn{2}{c|}{$\mathbf{\Abtil^{l}}$} & $\mathbf{\Abtil^{h}}$  & $\mathbf{\Abtil^{l}}$  \\
    & $c=1$ & $c=.8$ & $c=1$ & $c=.8$ & & \\
        \hline
        $\mathbf{10}$ & $46 \%$ & $0 \%$ & $54 \%$ & $0 \%$  & $.02$ & $.1$  \\
        \hline
        $\mathbf{100}$ & $0 \%$ & $0 \%$ & $35 \%$ & $0 \%$  & N/A & $.1$  \\
        \hline
        $\mathbf{1,000}$ & $100 \%$ & $0 \%$ &  $75 \%$ & $2 \%$ & .05 & $.2$ \\
        \hline
    \end{tabular}
\end{table}


\paragraph{Conclusions from our experimental evaluations}

We want to emphasize that the theoretical bound used in our experiments is not identical to the one presented in Theorem~\ref{lemma:general_rounding}, but rather our conjecture on what the true lower bound should be (see discussion after eqn.~\ref{eq:maineqapprox}). The rationale behind this choice is that the precise bound is overly pessimistic for modest values of the aspect ratio $\nicefrac{n}{d}$. This is probably due to state-of-the-art RMT bounds, which typically require much larger and impractical values of $n$. We show that, in practice, modifying the bound by a small constant (i.e., reducing $\mathcal{R} \sqrt{n \nu}$ to $c\cdot \mathcal{R} \sqrt{n \nu}$, with $c \geq.8$) results in excellent behavior in all our experimental settings.

\section{Future work} \label{section:future_work}
First, we need to relax the assumptions for the singular value bound in Theorem~\ref{lemma:general_rounding},  so they resemble the assumptions of Theorem~\ref{thm:gaussian}, where the perturbations are Gaussian. While we don't expect the gap between bounds for SR-nearness and Gaussian perturbations to be completely bridged, we need to understand how the former bounds can be improved. This will require novel and more powerful RMT results along the lines of Theorem~\ref{thm:inhomogeneous_anticoncentration}.

Second, we conjecture that for all sufficiently tall-and-thin matrices $\Ab$, SR-nearness produces a matrix whose smallest singular value  is bounded away from zero as in~(\ref{eq:maineqapprox}). Although removal of $\epsilon_{n,d}$ from (\ref{eq:maineqapprox}) seems infeasible with state-of-the-art RMT bounds, our numerical experiments strongly support this conjecture. 




\clearpage

\printbibliography

@ARTICLE{8766229,
  author={},
  journal={{IEEE Std 754-2019 (Revision of IEEE 754-2008)}}, 
  title={{IEEE Standard for Floating-Point Arithmetic}}, 
  year={2019},
  volume={},
  number={},
  pages={1-84},
  doi={10.1109/IEEESTD.2019.8766229}
}

@article{Rudelson2009,
author = {Rudelson, Mark and Vershynin, Roman},
doi = {10.1002/cpa.20294},
issn = {0010-3640},
journal = {Comm. Pure Appl. Math.},
month = {dec},
number = {12},
pages = {1707--1739},
title = {{Smallest singular value of a random rectangular matrix}},
url = {https://onlinelibrary.wiley.com/doi/10.1002/cpa.20294},
volume = {62},
year = {2009}
}

@article{Hoeff63,
 ISSN = {01621459},
 URL = {http://www.jstor.org/stable/2282952},
 author = {Wassily Hoeffding},
 journal = {Journal of the American Statistical Association},
 number = {301},
 pages = {13--30},
 publisher = {[American Statistical Association, Taylor & Francis, Ltd.]},
 title = {Probability Inequalities for Sums of Bounded Random Variables},
 urldate = {2024-02-18},
 volume = {58},
 year = {1963}
}

@article{MAL-048,
year = {2015},
volume = {8},
journal = {Foundations and Trends® in Machine Learning},
title = {An Introduction to Matrix Concentration Inequalities},
doi = {10.1561/2200000048},
issn = {1935-8237},
number = {1-2},
pages = {1-230},
author = {Joel A. Tropp}
}

@article{Boutsikas2024,
  title={Small singular values can increase in lower precision},
  author={Boutsikas, Christos and Drineas, Petros and Ipsen, Ilse CF},
  journal={SIAM Journal on Matrix Analysis and Applications},
  volume={45},
  number={3},
  pages={1518--1540},
  year={2024},
  publisher={SIAM}
}

@article{Forsythe1950,
author = {Forsythe, GE},
doi = {10.1090/S0002-9904-1950-09343-4},
issn = {0002-9904},
journal = {Bull. Amer. Math. Soc.},
month = {jan},
number = {1},
pages = {55--65},
title = {{Round-off errors in numerical integration on automatic machinery}},
url = {http://www.ams.org/journal-getitem?pii=S0002-9904-1950-09343-4},
volume = {56},
year = {1950}
}

@book{Horn2012,
author = {Horn, Roger A. and Johnson, Charles R.},
doi = {10.1017/CBO9781139020411},
isbn = {9780521839402},
month = {oct},
publisher = {Cambridge University Press},
title = {{Matrix Analysis}},
url = {https://www.cambridge.org/highereducation/product/9781139020411/book},
year = {2012}
}

@article{Croci2022,
journal = {Royal Society Open Science},
author = {Croci, Matteo and Fasi, Massimiliano and Higham, Nicholas J. and Mary, Theo and Mikaitis, Mantas},
month = {mar},
number = {3},
title = {{Stochastic rounding: implementation, error analysis and applications}},
url = {https://royalsocietypublishing.org/doi/10.1098/rsos.211631},
volume = {9},
year = {2022}
}

@article{dumitriu2022extreme,
  title={Extreme singular values of inhomogeneous sparse random rectangular matrices},
  author={Dumitriu, Ioana and Zhu, Yizhe},
  journal={arXiv preprint arXiv:2209.12271},
  year={2022}
}

@INPROCEEDINGS{Arar2022,
  author={Arar, El-Mehdi El and Sohier, Devan and de Oliveira Castro, Pablo and Petit, Eric},
  booktitle={2022 IEEE 29th Symposium on Computer Arithmetic (ARITH)}, 
  title={The Positive Effects of Stochastic Rounding in Numerical Algorithms}, 
  year={2022},
  volume={},
  number={},
  pages={58-65},
  doi={10.1109/ARITH54963.2022.00018}}

@article{von1947numerical,
author = {John von Neumann and H. H. Goldstine},
title = {{Numerical inverting of matrices of high order}},
volume = {53},
journal = {Bull. Amer. Math. Soc.},
number = {11},
publisher = {American Mathematical Society},
pages = {1021 -- 1099},
year = {1947},
}

@inproceedings{gupta2015deep,
  title={Deep learning with limited numerical precision},
  author={Gupta, Suyog and Agrawal, Ankur and Gopalakrishnan, Kailash and Narayanan, Pritish},
  booktitle={International conference on machine learning},
  pages={1737--1746},
  year={2015},
  organization={PMLR}
}

@article{wang2018training,
  title={Training deep neural networks with 8-bit floating point numbers},
  author={Wang, Naigang and Choi, Jungwook and Brand, Daniel and Chen, Chia-Yu and Gopalakrishnan, Kailash},
  journal={Advances in neural information processing systems},
  volume={31},
  year={2018}
}

@book{bhatia2013matrix,
  title={Matrix analysis},
  author={Bhatia, Rajendra},
  volume={169},
  year={2013},
  publisher={Springer Science \& Business Media}
}

@article{davies2018loihi,
  title={Loihi: A neuromorphic manycore processor with on-chip learning},
  author={Davies, Mike and Srinivasa, Narayan and Lin, Tsung-Han and Chinya, Gautham and Cao, Yongqiang and Choday, Sri Harsha and Dimou, Georgios and Joshi, Prasad and Imam, Nabil and Jain, Shweta and others},
  journal={IEEE Micro},
  volume={38},
  number={1},
  pages={82--99},
  year={2018},
  publisher={IEEE}
}

@misc{AMDpatent,
  title={Stochastic rounding logic},
  author={Gabriel H. Loh},
  url={https://patents.google.com/patent/US10628124B2/en?oq=US10628124B2},
  year={2018},
  publisher={Google Patents},
}

@misc{NVDIApatent,
  title={Stochastic rounding of numerical values},
  author={Jonah M. Alben and Paulius Micikevicius and Hao Wu and Ming Yiu Siu},
  url={https://patents.google.com/patent/US10684824B2/en?oq=US+10%2c684%2c824+B2},
  year={2019},
  publisher={Google Patents},
}

@misc{IBM1patent,
  title={Reproducible stochastic rounding for out of order processors},
  author={Jonathan D. Bradbury and Steven R. Carlough and Brian R. Prasky and Eric M. Schwarz},
  url={https://patents.google.com/patent/US10083008B2/en?oq=US+10083008+},
  year={2016},
  publisher={Google Patents},
}

@misc{IBM2patent,
  title={Stochastic rounding floating-point multiply instruction using entropy from a register},
  author={Jonathan D. Bradbury and Steven R. Carlough and Brian R. Prasky and Eric M. Schwarz},
  url={https://patents.google.com/patent/US10445066B2/en?oq=US+10445066},
  year={2016},
  publisher={Google Patents},
}

@misc{VIApatent,
  title={Processor with memory array operable as either cache memory or neural network unit memory},
  author={G. Glenn Henry and Douglas R. Reed},
  url={https://patents.google.com/patent/US10664751B2/en?oq=us+10664751},
  year={2016},
  publisher={Google Patents},
}

@misc{DensBitspatent,
  title={Apparatus and methods for hardware-efficient unbiased rounding},
  author={Ofir Avraham Kanter and Ilan Bar},
  url={https://patents.google.com/patent/US8972472B2/en?oq=US+8%2c972.472+},
  year={2008},
  publisher={Google Patents},
}

@misc{GSIpatent,
  title={In-memory stochastic rounder},
  author={Samuel Lifsches},
  url={https://patents.google.com/patent/US10803141B2/en?oq=10%2c803%2c141},
  year={2018},
  publisher={Google Patents},
}

@article{denis2015verificarlo,
  title={Verificarlo: Checking floating point accuracy through monte carlo arithmetic},
  author={Denis, Christophe and Castro, Pablo De Oliveira and Petit, Eric},
  journal={arXiv preprint arXiv:1509.01347},
  year={2015}
}

@article{jezequel2008cadna,
  title={CADNA: a library for estimating round-off error propagation},
  author={J{\'e}z{\'e}quel, Fabienne and Chesneaux, Jean-Marie},
  journal={Comput. Phys. Commun.},
  volume={178},
  number={12},
  pages={933--955},
  year={2008},
  publisher={Elsevier}
}

@article{fevotte2016verrou,
  title={{VERROU: a CESTAC evaluation without recompilation}},
  author={F{\'e}votte, Fran{\c{c}}ois and Lathuiliere, Bruno},
  journal={SCAN 2016},
  pages={47},
  year={2016}
}

@article{Sankar2006,
author = {Sankar, Arvind and Spielman, Daniel A. and Teng, Shang Hua},
doi = {10.1137/S0895479803436202},
journal = {SIAM J. Matrix Anal. Appl.},
number = {2},
pages = {446--476},
title = {{Smoothed analysis of the condition numbers and growth factors of matrices}},
volume = {28},
year = {2006}
}

@article{Arar2023,
author = {Arar, El-Mehdi El and Sohier, Devan and {de Oliveira Castro}, Pablo and Petit, Eric},
doi = {10.1137/22M1510819},
journal = {SIAM J. Sci. Comput.},
number = {5},
pages = {C255--C275},
title = {{Stochastic Rounding Variance and Probabilistic Bounds: A New Approach}},
url = {https://doi.org/10.1137/22M1510819},
volume = {45},
year = {2023}
}

@book{parker1997monte,
  title={Monte Carlo arithmetic: exploiting randomness in floating-point arithmetic},
  author={Parker, Douglass Stott},
  year={1997},
  publisher={Citeseer}
}

@article{croci2023effects,
  title={Effects of round-to-nearest and stochastic rounding in the numerical solution of the heat equation in low precision},
  author={Croci, Matteo and Giles, Michael B.},
  journal={IMA J. Numer. Anal.},
  volume={43},
  number={3},
  pages={1358--1390},
  year={2023},
  publisher={Oxford University Press}
}

@article{hull1966tests,
  title={Tests of probabilistic models for propagation of roundoff errors},
  author={Hull, Thomas E and Swenson, J Richard},
  journal={Commun. ACM},
  volume={9},
  number={2},
  pages={108--113},
  year={1966},
  publisher={ACM New York, NY, USA}
}

@article{connolly2021stochastic,
  title={Stochastic rounding and its probabilistic backward error analysis},
  author={Connolly, Michael P and Higham, Nicholas J and Mary, Theo},
  journal={SIAM J. Sci. Comput.},
  volume={43},
  number={1},
  pages={A566--A585},
  year={2021},
  publisher={SIAM}
}

@article{hallman2023precision,
  title={Precision-aware deterministic and probabilistic error bounds for floating point summation},
  author={Hallman, Eric and Ipsen, Ilse C. F.},
  journal={Numer. Math.},
  volume={155},
  number={1-2},
  pages={83--119},
  year={2023},
  publisher={Springer}
}

@article{vignes2004discrete,
  title={Discrete stochastic arithmetic for validating results of numerical software},
  author={Vignes, Jean},
  journal={Numer. Algorithms},
  volume={37},
  pages={377--390},
  year={2004},
  publisher={Springer}
}

@article{higham2019simulating,
  title={Simulating low precision floating-point arithmetic},
  author={Higham, Nicholas J and Pranesh, Srikara},
  journal={SIAM Journal on Scientific Computing},
  volume={41},
  number={5},
  pages={C585--C602},
  year={2019},
  publisher={SIAM}
}

@article{forsythe1959reprint,
  title={Reprint of a note on rounding-off errors},
  author={Forsythe, George E},
  journal={SIAM review},
  volume={1},
  number={1},
  pages={66},
  year={1959},
  publisher={Society for Industrial and Applied Mathematics}
}

@misc{P1,
  title = {{Mixed-Precision Arithmetic for AI: A Hardware Perspective}},
  howpublished = {\url{https://docs.graphcore.ai/projects/ai-float-white-paper/en/latest/ai-float.html}},
  note = {Accessed: 2024-02-20}
}

@article{dumitriu2024extremepub,
  title={Extreme singular values of inhomogeneous sparse random rectangular matrices},
  author={Dumitriu, Ioana and Zhu, Yizhe},
  journal={Bernoulli},
  volume={30},
  number={4},
  pages={2904--2931},
  year={2024},
  publisher={Bernoulli Society for Mathematical Statistics and Probability}
}

@article{brailovskaya2024universality,
  title={Universality and sharp matrix concentration inequalities},
  author={Brailovskaya, Tatiana and van Handel, Ramon},
  journal={Geometric and Functional Analysis},
  pages={1--105},
  year={2024},
  publisher={Springer}
}

\clearpage

\appendix

\section{Proof for Section~\ref{sxn:simpleexample}}\label{sxn:app:extra}
Distinguish the columns of the rounded matrix,
\begin{align*}
\Abtil=\begin{bmatrix} \Abtil_1 & \Abtil_2\end{bmatrix}.
\end{align*}
Since, in expectation, half of the entries of $\Abtil$ are equal to one and the other half are equal to zero,
the columns of $\Abtil$ have expected squared norms equal to 
\begin{align}
 \EE\|\Abtil_j\|_2^2 = \frac{n}{2}, \qquad j=1,2.
\end{align}
In the inner product between the two columns,
\begin{align*}
\Abtil_1^T \Abtil_2= \sum_{i=1}^n \Abtil_{i1}\Abtil_{i2},
\end{align*}
the $i$th summand $\Abtil_{i1}\Abtil_{i2}=1$ only if 
$\Ab_{i,1}=\Ab_{i,2}=1$, which occurs with probability~$1/4$.
Hence the expected inner product between the two columns equals 
\begin{align}
\EE (\Ab_1^T \Ab_2) = \frac{n}{4}.
\end{align}
The $2 \times 2$ Gram matrix $\Abtil^T \Abtil$ has expectation
\begin{align*}
    \EE [\Abtil^T \Abtil] = 
    \begin{bmatrix}
        \frac{n}{2} & \frac{n}{4} \\
        \frac{n}{4} & \frac{n}{2}
    \end{bmatrix}
    = \frac{n}{4} \,
    \begin{bmatrix}
        2 & 1 \\
        1 & 2
    \end{bmatrix},
\end{align*}
which immediately implies that 
$\sigma_2(\EE[\Abtil^T\Abtil])=\nicefrac{n}{4}$. 

Denote by $\Bb = \Abtil^T\Abtil - \EE[\Abtil^T\Abtil]$
the deviation of the Gram matrix from its expectation.
Weyl's inequality and the bound $\|\Bb\|_2 \leq \|\Bb\|_F$ give
\begin{align}
    \sigma_2(\Abtil)^2 = \sigma_2(\Abtil^T\Abtil) &= \sigma_2(\EE[\Abtil^T\Abtil] - (\EE[\Abtil^T\Abtil]-\Abtil^T\Abtil)) \nonumber\\
    &\geq \sigma_2(\EE[\Abtil^T\Abtil]) - \|\Bb\|_2\nonumber\\
    &\geq \sigma_2(\EE[\Abtil^T\Abtil]) - \|\Bb\|_F\nonumber\\
    &= \frac{n}{4} - \|\Bb\|_F.\label{eqn:pd456}
\end{align}
It remains to bound $\|\Bb\|_F$.
Since $(\Abtil^T\Abtil)_{ij} = \sum_{k=1}^n \Abtil_{ki}\Abtil_{kj}$ is the sum of $n$ random variables that are either zero or one, we can invoke Theorem~\ref{thm:hoeffding},
\begin{align*}
    \PP\left[|\Bb_{ij}| \geq t\right])=
    \PP\left[|(\Abtil^T\Abtil)_{ij} - \EE[(\Abtil^T\Abtil)_{ij}] \geq t\right] \leq 2\exp\left(\frac{-2t^2}{n}\right).
\end{align*}
Applying a union bound over the four events that 
represent the entries of $\Bb$ being less than $t$ gives the 
failure probability 
\begin{align*}
    \PP\left[\sum_{i,j=1}^{2} |\Bb_{ij}| \geq 4t \right] &\leq 8\exp\left(\frac{-2t^2}{n}\right).
    \end{align*}
    and the success probability of the complementary event,
    \begin{align*}
    \PP\left[\sum_{i,j=1}^{2} |\Bb_{ij}| \leq 4t \right] &\geq 1-8\exp\left(\frac{-2t^2}{n}\right).
\end{align*}
Vector $p$-norm relations imply
\begin{align*}
\|\Bb\|_F=\|\mathrm{vec}(\Bb)\|_2\leq \|\mathrm{vec}(\Bb)\|_1
=\sum_{i,j=1}^2{|\Bb_{ij}|}.
\end{align*}
Insert this into the success probability,
\begin{align*}
    \PP\left(\|\Bb\|_F \leq 4t \right) \geq 1-8\exp\left(\frac{-2t^2}{n}\right).
\end{align*}
and combine with (\ref{eqn:pd456}) to obtain a lower bound for $\sigma_2^2(\Abtil)$,
\begin{align*}
    \PP\left(\sigma_2^2(\Abtil) \geq \frac{n}{4} - 4t\right) \geq 1-8\exp\left(\frac{-2t^2}{n}\right).
\end{align*}
Setting $t = 2\sqrt{n}$ gives
\begin{align}
 \sigma_2^2(\Abtil) \geq 0.25 \cdot n - 8\sqrt{n},
\end{align}
with probability at least $0.997$.

\section{Gaussian perturbations}\label{app_G}
We consider perturbations $\Abtil = \Ab + \Eb$, where\footnote{By rescaling the bound by $\sigma > 0$, we can extend the bound to any $\Eb$ with $\Eb_{ij} = {\cal N}(0,\sigma)$.}
$\Eb_{i,j} = {\cal N}(0,1)$, and show
that the smallest singular value of $\Abtil$ is bounded away from zero with high probability. While this 
perturbation model is not relevant for stochastic rounding, we do note that Theorem~\ref{thm:gaussian}
is much sharper than Theorem~\ref{lemma:general_rounding} and Corollary~\ref{cor:general_rounding}. 

For example, if $n = 900$ and $d = 25$, then Theorem~\ref{thm:gaussian}, with $t=4$ shows that $\sigma_d(\Abtil) \geq 1$ with probability at least $0.98$, \textit{without any additional assumptions}. This provides evidence that the assumptions of Theorem~\ref{lemma:general_rounding} and Corollary~\ref{cor:general_rounding} could be significantly relaxed for non-Gaussian, non-identically distributed perturbations. The stronger bounds in this section are derived to from very strong measure concentration inequalities for Gaussian distributions, as well as the fact that these distribution are invariant under unitary transformations. 

The proof follows the same high-level proof structure as in Section~\ref{sec_proof}. We start by stating our main result.
\begin{theorem}\label{thm:gaussian}
    Let $\Ab \in \R^{n \times d}$, and let $\Eb \in \R^{n \times d}$ be a random matrix with independent, identically distributed Gaussian entries, i.e., $\Eb_{ij}={\cal N}(0,1)$. Then, for any $t > 0$,
    \begin{align*}
        \PP\left(\sigma_d(\Ab + \Eb) \geq \sqrt{n} - (1+t)\sqrt{d} - t \right)
        \geq 1-(2d+1)e^{-t^2/2}.
    \end{align*}
\end{theorem}

\begin{proof}
    As in Section \ref{sec_proof}, we decompose the task of lower bounding $\sigma_d(\Ab + \Eb)$ into two parts. Again, for any orthogonal projector $\Pb$,
    \begin{align*}
    \sigma_{d}(\Ab+\Eb) \geq \sigma_{d}(\Pb(\Ab+\Eb)).
    \end{align*}
Let $\Pb_{\Ab} \in \R^{n \times n}$ be the orthogonal projector onto the $d$ dimensional column space of~$\Ab$, 
    and $\Pbb=\Ib-\Pb_{\Ab}$ the orthogonal projector onto the $n - d$ dimensional left null space of $\Ab$.
        Then 
    \begin{align*}
        \sigma_d(\Ab + \Eb) \geq \sigma_d(\Pbb(\Ab + \Eb)) = \sigma_d(\Pbb\Eb).
    \end{align*}
    Weyl's inequality implies
    \begin{align}\label{eqn:ppd112}
        \sigma_d(\Pbb\Eb) = \sigma_d(\Eb - (\Ib - \Pbb)\Eb) \geq \sigma_d(\Eb) - \|\Pb_{\Ab}\Eb\|_2.
    \end{align}
    We have now broken our task into two conceptual components. First, we must make sure that the random matrix $\Eb$ has a large minimum singular value. Second, the projection of $\Eb$  must be small. In other words, the matrix $\Eb$ should not concentrate in any $d$-dimensional space.
    
    We start by lower bounding $\sigma_d(\Eb)$. 
    From \cite[Expression below (1.11)]{Rudelson2009}
    follows that for any $t>0$,
    %
    \begin{align}\label{eqn:ppd113}
        \PP\big(\sigma_d(\Eb) \leq \sqrt{n} - \sqrt{d} - t\big) \leq e^{-t^2/2}.
    \end{align}
    Next, we bound $\|\Pb_{\Ab}\Eb\|_2$ by first leveraging the unitary invariance of the Gaussian distribution. If $\gb \in \R^{n}$ is a vector with independent, identically distributed standard normal entries, then the distribution of $\Ub\gb$ is the same as the distribution of $\gb$ for any orthogonal matrix $\Ub\in\R^{n\times n}$.
    %
    %
    Since the columns of $\Eb$ are independent, we can apply this result column-wise to conclude that the distribution of $\Eb$ is equal to the distribution of $\Ub\Eb$ for any orthogonal matrix $\Ub$.
    Let 
    \begin{align}\label{eqn:projection}
     \Pb_{\Ab} = \Ub
     \left[ \begin{array}{cc}
     \Ib_d &\zero_{d \times (n-d)}\\
     \zero_{(n-d)\times d} & \zero_{(n-d)\times (n-d)} \\
     \end{array}\right]
     \Ub^T
    \end{align}
    be an eigenvalue decomposition where  $\Ub\in\R^{n\times n}$ is an orthogonal matrix. 
    From the rotational invariance of the Gaussian distribution follows that the entries of 
    $\Ub^T\Eb\in\R^{n\times d}$  are also independent, identically distributed Gaussian normal random variables. The unitary invariance of the two-norm implies
    \begin{align}
     \|\Pb_{\Ab}\Eb\|_2 = \|\Gb\|_2 \qquad \text{where}\qquad
     %
     %
     \Gb\equiv \begin{bmatrix}
     \Ib_d &\zero_{d \times (n-d)}\end{bmatrix}\Ub^T\Eb
     \in\R^{d\times d}
    %
\label{eqn:ppd111}
    \end{align}
is the matrix that contains the leading $d$
    rows of $\Ub^T\Eb$.
    
    %
    %
    %

    From Lemma~\ref{lemma_gaussian_spectral_bound} and (\ref{eqn:ppd111}) follows
    \begin{align*}
     \PP\big(\|\Pb_{\Ab}\Eb\|_2 = \PP\big(\|\Gb\|_2 \geq t\sqrt{d} \big) \leq 2d e^{-t^2/2}   
    \end{align*}
    for all $t>0$. Combining this with (\ref{eqn:ppd113}) and applying a union bound to control the two failure probabilities gives
    \begin{align*}
        \PP\left(\sigma_d(\Eb) - \|\Pb_{\Ab} \Eb\|_2 \leq \sqrt{n} - (1+t)\sqrt{d} - t \right)
        \leq (2d+1)e^{-t^2/2}
    \end{align*}
    for all $t>0$. The complement of the above event, for all $t>0$, is
    \begin{align*}
        \PP\left(\sigma_d(\Eb) - \|\Pb_{\Ab} \Eb\|_2 \geq \sqrt{n} - (1+t)\sqrt{d} - t \right) \geq 
        1-(2d+1)e^{-t^2/2}.
    \end{align*}
    At last, combine the above with (\ref{eqn:ppd112}).
\end{proof}

In order to bound the largest singular value of a Gaussian matrix, we will use the following concentration inequality from prior work.

\begin{theorem}\label{thm:tropp_gaussian}
    (Theorem 4.1.1 in \cite{MAL-048}) Consider a finite sequence of $\{\Bb_k\}$ of fixed real-valued matrices with dimension $d_1 \times d_2$, and let $\{\gamma_k\}$ be a finite sequence of independent standard normal variables. Introduce the Gaussian series:
    \begin{gather*}
        \Zb = \sum_k \gamma_k \Bb_k.
    \end{gather*}
    Let ${\mathcal V}(\Zb)$ be the matrix variance statistic of the sum:
    \begin{align*}
        {\mathcal V}(\Zb) &= \max\{\EE \|\Zb\Zb^T\|_2, \EE\|\Zb^T\Zb\|_2\} \\
        &= \max\big\{\EE \big\|\sum_k \Bb_k\Bb_k^T\big\|_2, \EE\big\|\sum_k \Bb_k^T\Bb_k\big\|_2\big\}.
    \end{align*}
    Then, 
    \begin{gather*}
        \EE\|\Zb\|_2 \leq \sqrt{2\cdot {\mathcal V}(\Zb)\log(d_1 + d_2)}.
    \end{gather*}
    Furthermore, for all $t \geq 0$,
    \begin{gather*}
        \PP(\|\Zb\|_2 \geq t) \leq (d_1 + d_2) \exp\left( \frac{-t^2}{2\cdot {\mathcal V}(\Zb)} \right).
    \end{gather*}
\end{theorem}

The following lemma bounds the largest singular value of Gaussian matrices, whose entries are independent, identically distributed ${\cal N}(0,1)$ random variables.

\begin{lemma}\label{lemma_gaussian_spectral_bound}
        Let $\Gb \in \R^{d \times d}$ be a random matrix with independent, identically distributed standard normal entries. Then, for all $t \geq 0$,
        \begin{align*}
            \PP\big(\|\Gb\|_2 \geq t\sqrt{d} \big) \leq 2d e^{-t^2/2}.
        \end{align*}
    \end{lemma}

\begin{proof}
We will apply Theorem \ref{thm:tropp_gaussian}. 
First write $\Gb$ as a sum of outer products
\begin{align*}
    \Gb = \sum_{i,j=1}^d g_{ij} \eb_i \eb_j^T,
\end{align*}
where $g_{ij}$ are independent identically distributed standard normal variables and $\eb_i \in \mathbb{R}^d$, $1\leq i=1\leq d$, are the columns of the identity matrix
$\Ib\in\R^{d\times d}$.
In order to apply the aforementioned theorem we need to bound:
\begin{align*}
   \left \|\sum_{i,j=1}^d \eb_i \eb_j^T(\eb_i \eb_j^T)^T\right\|_2 = 
   \left\|\sum_{i,j=1}^d \eb_i \eb_i^T \right\|_2 =  \|d \cdot\Ib_d\|_2 = d.
\end{align*}
Similarly, $\|\sum_{i,j=1}^d (\eb_i \eb_j^T)^T\eb_i \eb_j^T\|_2 = d$. Hence in Theorem \ref{thm:tropp_gaussian} the parameter ${\mathcal V}(\Gb)$ is equal to $d$, and so,
\begin{align*}
\PP\big[\|\Gb\|_2 \geq t \big] \leq 2d e^{-t^2/2d}.
\end{align*}
Rescaling the parameter $t$ by a $\sqrt{d}$ factor concludes the proof.
\end{proof}

\section{Additional plots} \label{append_plots}

\begin{figure}[ht!]
    \begin{subfigure}[t]{0.45\textwidth}
    \centering
    \includegraphics[page=1, width=\linewidth]{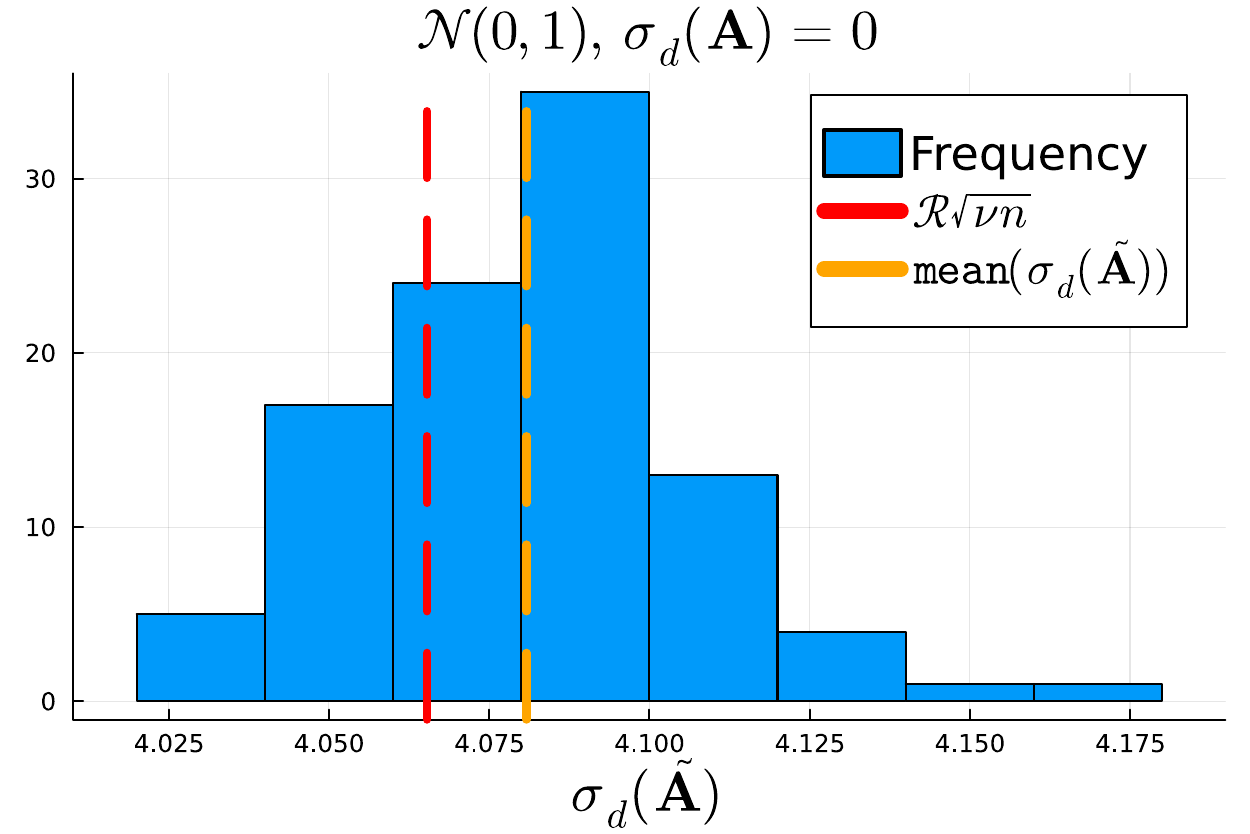}
    \caption{$p=1, \, d = 10$} \label{exp_fixed_rank_def_plot_1}
    \end{subfigure}
    \hfill
    \centering
    \begin{subfigure}[t]{0.45\textwidth}
        \centering
        \includegraphics[page=11, width=\linewidth]{new_plots/randn_10e4_0_all_reps_.pdf}
    \caption{$p=1, \, d = 1000$}  \label{exp_fixed_rank_def_plot_2}
    \end{subfigure}

    \begin{subfigure}[t]{0.45\textwidth}
    \centering
    \includegraphics[page=2, width=\linewidth]{new_plots/randn_10e4_0_all_reps_.pdf}
    \caption{$p=2, \, d = 10$}  \label{exp_fixed_rank_def_plot_3}
    \end{subfigure}
    \hfill
    \centering
    \begin{subfigure}[t]{0.45\textwidth}
        \centering
        \includegraphics[page=12, width=\linewidth]{new_plots/randn_10e4_0_all_reps_.pdf}
    \caption{$p=2, \, d = 1000$} \label{exp_fixed_rank_def_plot_4}
    \end{subfigure}

     \begin{subfigure}[t]{0.45\textwidth}
    \centering
    \includegraphics[page=3, width=\linewidth]{new_plots/randn_10e4_0_all_reps_.pdf}
    \caption{$p=3, \, d = 10$}  \label{exp_fixed_rank_def_plot_5}
    \end{subfigure}
    \hfill
    \centering
    \begin{subfigure}[t]{0.45\textwidth}
        \centering
        \includegraphics[page=13, width=\linewidth]{new_plots/randn_10e4_0_all_reps_.pdf}
    \caption{$p=3, \, d = 1000$} \label{exp_fixed_rank_def_plot_6}
    \end{subfigure}

\caption{The elements of  $\Ab$ are random variables in
$\mathcal{N}(0,1)$ with $\sigma_d(\Ab)=0$. The stochastically rounded $\Abtil$ 
has elements in $\mathcal{F}^{{p}}$, for $p = 1, 2, 3$. The horizontal axis represents the values of $\sigma_d(\Abtil)$ over 100 runs, grouped into at most 10 bins. The vertical axis represents the number of $\sigma_d(\Abtil)$ in each bin. The
\textcolor{orange}{orange dashed vertical line} represents the average value of $\sigma_d(\Abtil)$, while the \textcolor{red}{red dashed vertical line} represents the lower bound estimate~(\ref{eq:maineqapprox}). 
Each panel corresponds to a different combination of $p$ and $d$. In each row, the precision~$p$ is fixed, while the column dimension~$d$ varies.} \label{exp_fixed_rank_def_plot}
\end{figure}

\begin{figure}[ht!]
    \begin{subfigure}[t]{0.45\textwidth}
    \centering
    \includegraphics[page=1, width=\linewidth]{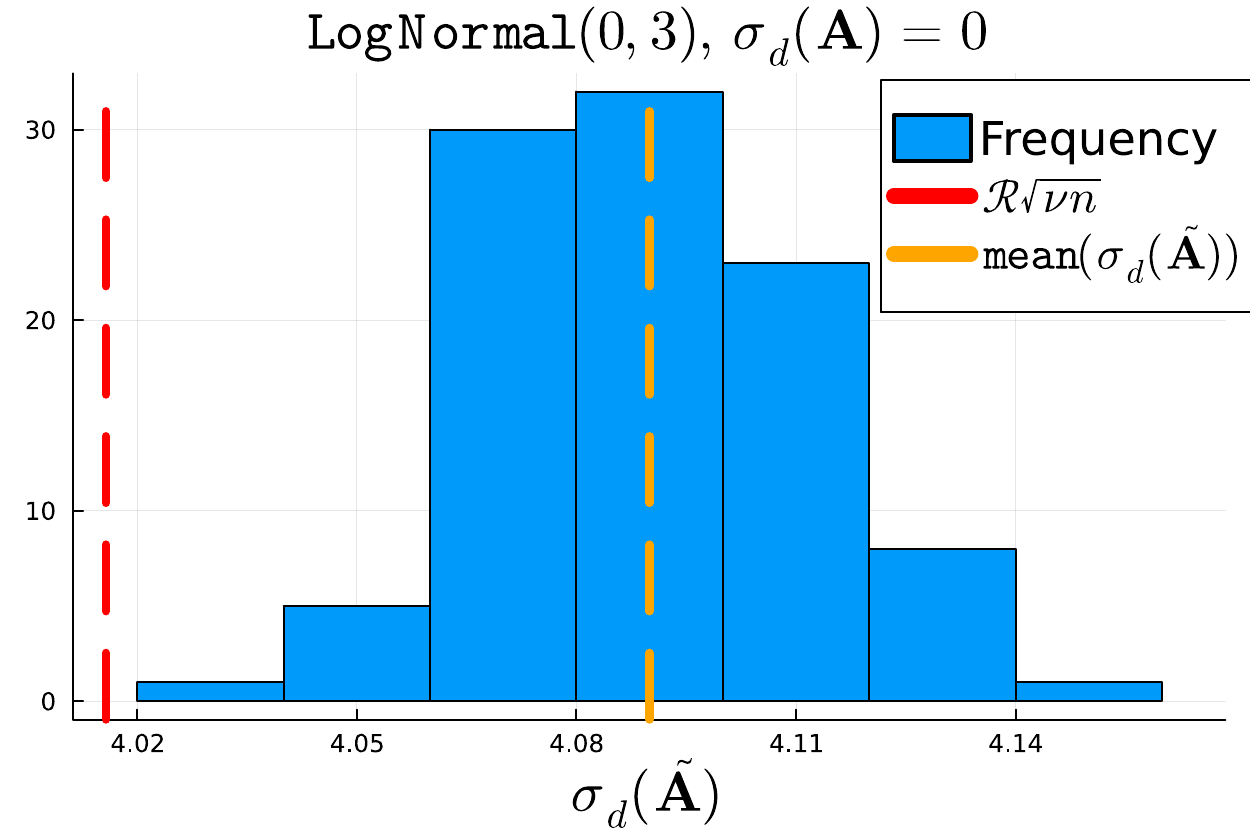}
    \caption{$p=1, \, d = 10$} \label{exp_fixed_rank_def_lognorm_plot_1}
    \end{subfigure}
    \hfill
    \centering
    \begin{subfigure}[t]{0.45\textwidth}
        \centering
        \includegraphics[page=11, width=\linewidth]{new_plots/lognorm_10e4_0_all_reps_mean_0_var_3_.pdf}
    \caption{$p=1, \, d = 1000$}  \label{exp_fixed_rank_def_lognorm_plot_2}
    \end{subfigure}

    \begin{subfigure}[t]{0.45\textwidth}
    \centering
    \includegraphics[page=2, width=\linewidth]{new_plots/lognorm_10e4_0_all_reps_mean_0_var_3_.pdf}
    \caption{$p=2, \, d = 10$}  \label{exp_fixed_rank_def_lognorm_plot_3}
    \end{subfigure}
    \hfill
    \centering
    \begin{subfigure}[t]{0.45\textwidth}
        \centering
        \includegraphics[page=12, width=\linewidth]{new_plots/lognorm_10e4_0_all_reps_mean_0_var_3_.pdf}
    \caption{$p=2, \, d = 1000$} \label{exp_fixed_rank_def_lognorm_plot_4}
    \end{subfigure}

     \begin{subfigure}[t]{0.45\textwidth}
    \centering
    \includegraphics[page=3, width=\linewidth]{new_plots/lognorm_10e4_0_all_reps_mean_0_var_3_.pdf}
    \caption{$p=3, \, d = 10$}  \label{exp_fixed_rank_def_lognorm_plot_5}
    \end{subfigure}
    \hfill
    \centering
    \begin{subfigure}[t]{0.45\textwidth}
        \centering
        \includegraphics[page=13, width=\linewidth]{new_plots/lognorm_10e4_0_all_reps_mean_0_var_3_.pdf}
    \caption{$p=3, \, d = 1000$} \label{exp_fixed_rank_def_lognorm_plot_6}
    \end{subfigure}

\caption{The matrices are initially drawn from a log-normal distribution with the smallest singular value set to 0, and stochastically rounded to $\mathcal{F}^{{p}}$, for $p = 1, \ldots, 3$. The horizontal axis represents the distribution of $\sigma_d(\Abtil)$ over 100 repetitions, grouped into up to 10 bins. The vertical axis shows the frequency with which each $\sigma_d(\Abtil)$ appears in each bin. The
\textcolor{orange}{orange dashed vertical line} represents the average value of $\sigma_d(\Abtil)$, while the \textcolor{red}{red dashed vertical line} represents the lower bound estimate~(\ref{eq:maineqapprox}).
Each panel corresponds to a different combination of $p$ and $d$. In each row, the precision~$p$ is fixed, while the column dimension~$d$ varies.} \label{exp_fixed_rank_def_lognorm_plot}

\end{figure}

\begin{figure}[ht!]
    \begin{subfigure}[t]{0.45\textwidth}
    \centering
    \includegraphics[page=1, width=\linewidth]{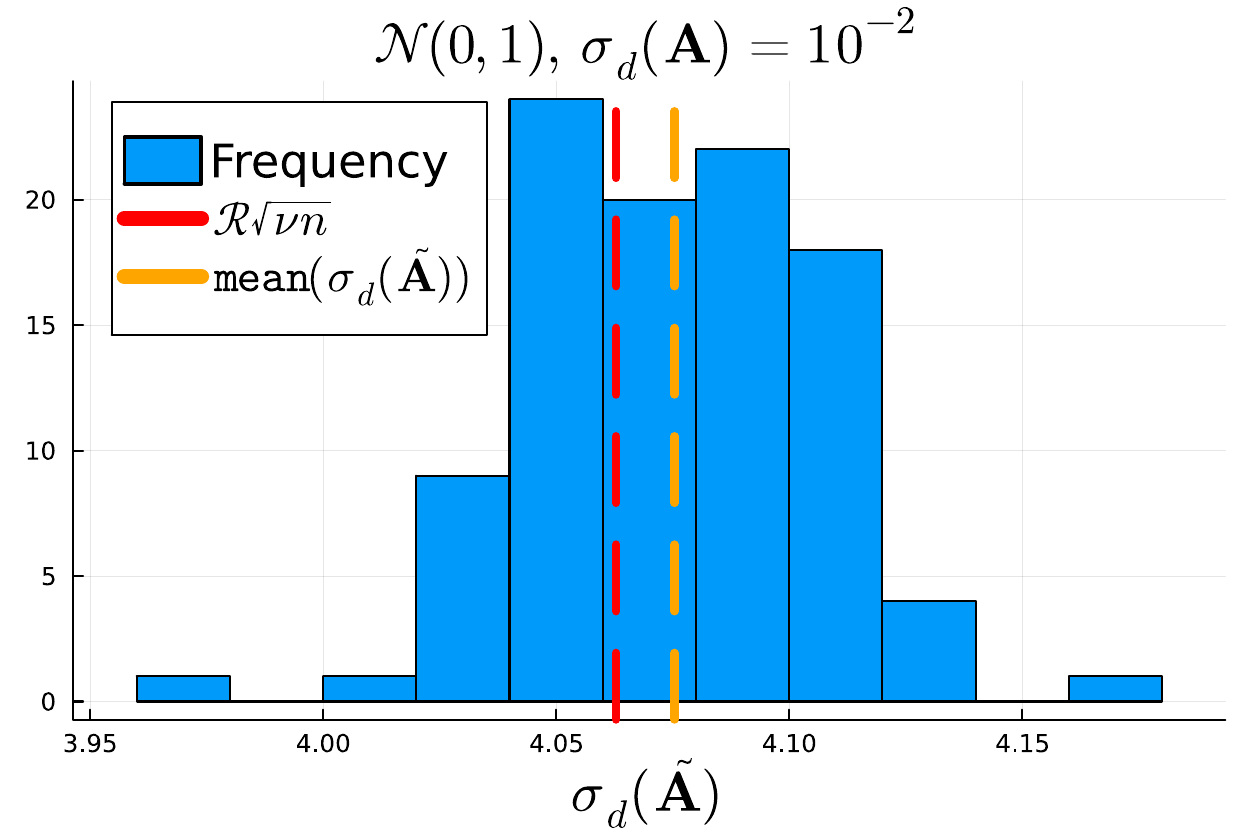}
    \caption{$p=1, \, d = 10$} \label{exp_fixed_rank_full_1}
    \end{subfigure}
    \hfill
    \centering
    \begin{subfigure}[t]{0.45\textwidth}
        \centering
        \includegraphics[page=11, width=\linewidth]{new_plots/randn_10e4_10e-2_all_reps_.pdf}
    \caption{$p=1, \, d = 1000$}  \label{exp_fixed_rank_full_2}
    \end{subfigure}

    \begin{subfigure}[t]{0.45\textwidth}
    \centering
    \includegraphics[page=2, width=\linewidth]{new_plots/randn_10e4_10e-2_all_reps_.pdf}
    \caption{$p=2, \, d = 10$}  \label{exp_fixed_rank_full_3}
    \end{subfigure}
    \hfill
    \centering
    \begin{subfigure}[t]{0.45\textwidth}
        \centering
        \includegraphics[page=12, width=\linewidth]{new_plots/randn_10e4_10e-2_all_reps_.pdf}
    \caption{$p=2, \, d = 1000$} \label{exp_fixed_rank_full_4}
    \end{subfigure}

     \begin{subfigure}[t]{0.45\textwidth}
    \centering
    \includegraphics[page=3, width=\linewidth]{new_plots/randn_10e4_10e-2_all_reps_.pdf}
    \caption{$p=3, \, d = 10$}  \label{exp_fixed_rank_full_5}
    \end{subfigure}
    \hfill
    \centering
    \begin{subfigure}[t]{0.45\textwidth}
        \centering
        \includegraphics[page=13, width=\linewidth]{new_plots/randn_10e4_10e-2_all_reps_.pdf}
    \caption{$p=3, \, d = 1000$} \label{exp_fixed_rank_full_6}
    \end{subfigure}

\caption{The matrices are initially drawn from a standard normal distribution with the smallest singular value set to $10^{-2}$, and stochastically rounded to $\mathcal{F}^{{p}}$, for $p = 1, \ldots, 3$. The horizontal axis represents the distribution of $\sigma_d(\Abtil)$ over 100 repetitions, grouped into up to 10 bins. The vertical axis shows the frequency with which each $\sigma_d(\Abtil)$ appears in each bin. The \textcolor{orange}{orange dashed vertical line} represents the average value of $\sigma_d(\Abtil)$, while the \textcolor{red}{red dashed vertical line} represents the lower bound estimate~(\ref{eq:maineqapprox}). 
Each panel corresponds to a different combination of $p$ and $d$. In each row, the precision~$p$ is fixed, while the column dimension~$d$ varies.} \label{exp_fixed_rank_full_plot}

\end{figure}

\begin{figure}[ht!]
    \begin{subfigure}[t]{0.45\textwidth}
    \centering
    \includegraphics[page=1, width=\linewidth]{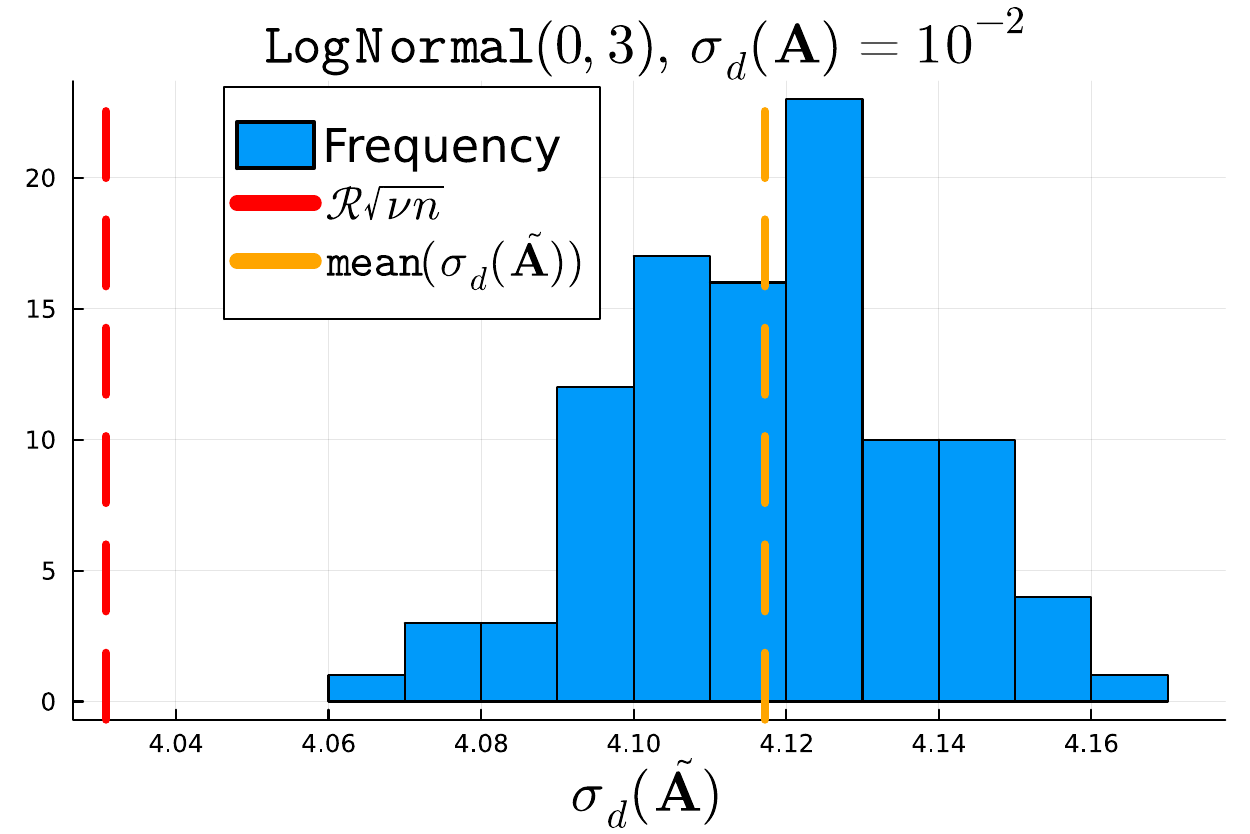}
    \caption{$p=1, \, d = 10$} \label{exp_fixed_rank_full_lognorm_plot_1}
    \end{subfigure}
    \hfill
    \centering
    \begin{subfigure}[t]{0.45\textwidth}
        \centering
        \includegraphics[page=11, width=\linewidth]{new_plots/lognorm_10e4_10e-2_all_reps_mean_0_var_3_.pdf}
    \caption{$p=1, \, d = 1000$}  \label{exp_fixed_rank_full_lognorm_plot_2}
    \end{subfigure}

    \begin{subfigure}[t]{0.45\textwidth}
    \centering
    \includegraphics[page=2, width=\linewidth]{new_plots/lognorm_10e4_10e-2_all_reps_mean_0_var_3_.pdf}
    \caption{$p=2, \, d = 10$}  \label{exp_fixed_rank_full_lognorm_plot_3}
    \end{subfigure}
    \hfill
    \centering
    \begin{subfigure}[t]{0.45\textwidth}
        \centering
        \includegraphics[page=12, width=\linewidth]{new_plots/lognorm_10e4_10e-2_all_reps_mean_0_var_3_.pdf}
    \caption{$p=2, \, d = 1000$} \label{exp_fixed_rank_full_lognorm_plot_4}
    \end{subfigure}

     \begin{subfigure}[t]{0.45\textwidth}
    \centering
    \includegraphics[page=3, width=\linewidth]{new_plots/lognorm_10e4_10e-2_all_reps_mean_0_var_3_.pdf}
    \caption{$p=3, \, d = 10$}  \label{exp_fixed_rank_full_lognorm_plot_5}
    \end{subfigure}
    \hfill
    \centering
    \begin{subfigure}[t]{0.45\textwidth}
        \centering
        \includegraphics[page=13, width=\linewidth]{new_plots/lognorm_10e4_10e-2_all_reps_mean_0_var_3_.pdf}
    \caption{$p=3, \, d = 1000$} \label{exp_fixed_rank_full_lognorm_plot_6}
    \end{subfigure}

\caption{The matrices are initially drawn from a log-normal distribution with the smallest singular value set to $10^{-2}$, and stochastically rounded to $\mathcal{F}^{{p}}$, for $p = 1, \ldots, 3$. The horizontal axis represents the distribution of $\sigma_d(\Abtil)$ over 100 repetitions, grouped into up to 10 bins. The vertical axis shows the frequency with which each $\sigma_d(\Abtil)$ appears in each bin. The \textcolor{orange}{orange dashed vertical line} represents the average value of $\sigma_d(\Abtil)$, while the \textcolor{red}{red dashed vertical line} represents the lower bound estimate~(\ref{eq:maineqapprox}). 
Each panel corresponds to a different combination of $p$ and $d$. In each row, the precision~$p$ is fixed, while the column dimension~$d$ varies.} \label{exp_fixed_rank_full_lognorm_plot}

\end{figure}

\begin{figure}[ht!]
    \begin{subfigure}[t]{0.45\textwidth}
    \centering
    \includegraphics[page=1, width=\linewidth]{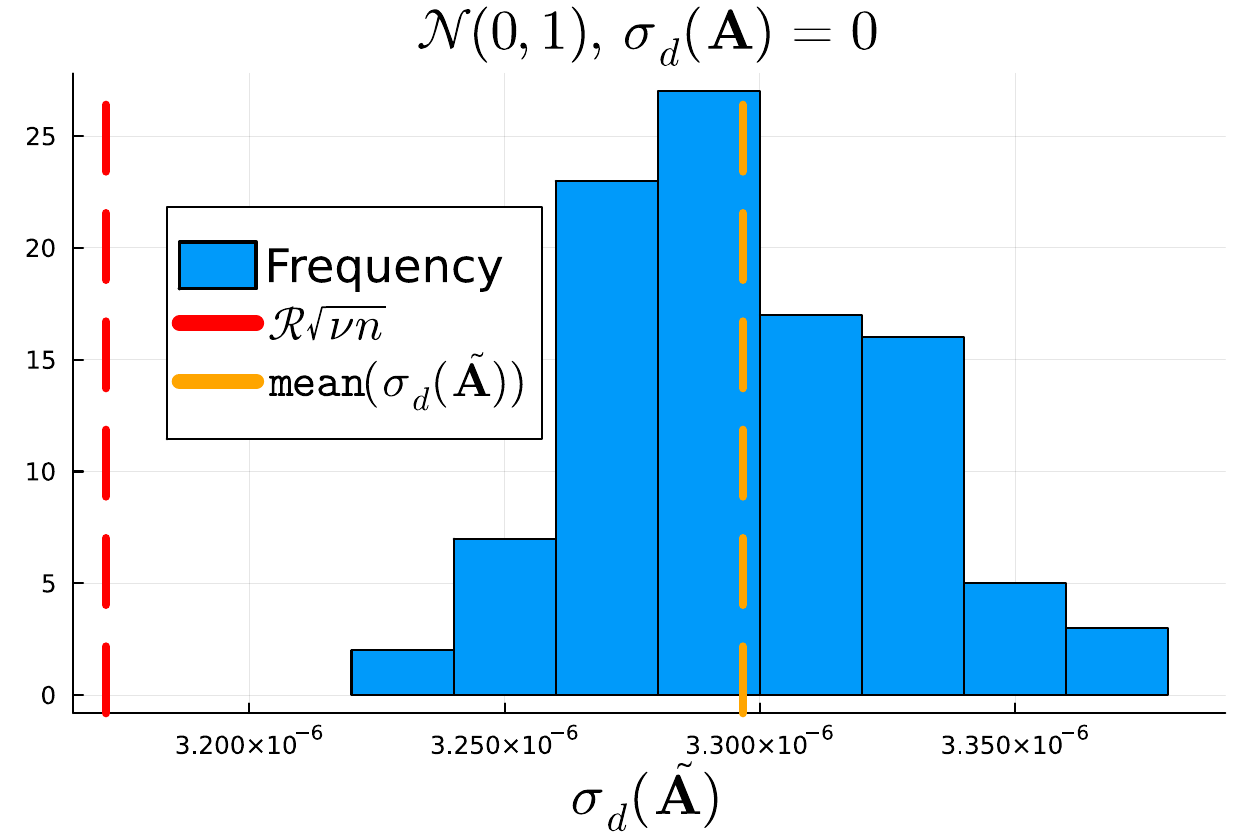}
    \caption{$d = 10$} \label{exp_fl_rank_def_plot_1}
    \end{subfigure}
    \hfill
    \centering
    \begin{subfigure}[t]{0.45\textwidth}
        \centering
        \includegraphics[page=1, width=\linewidth]{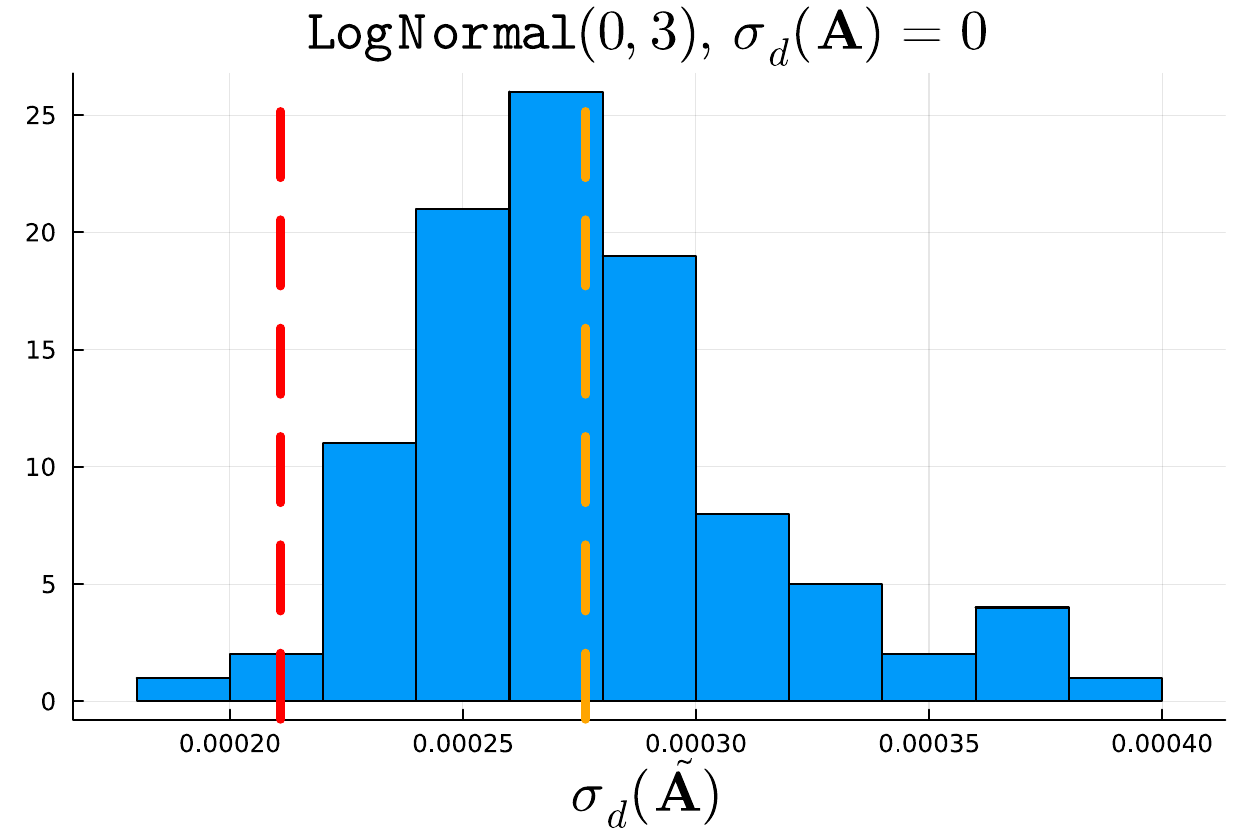}
    \caption{$d = 10$}  \label{exp_fl_rank_def_plot_2}
    \end{subfigure}

    \begin{subfigure}[t]{0.45\textwidth}
    \centering
    \includegraphics[page=2, width=\linewidth]{new_plots/randn_fl.pdf}
    \caption{$d = 100$}  \label{exp_fl_rank_def_plot_3}
    \end{subfigure}
    \hfill
    \centering
    \begin{subfigure}[t]{0.45\textwidth}
        \centering
        \includegraphics[page=2, width=\linewidth]{new_plots/randn_fl_lognorm.pdf}
    \caption{$d = 100$} \label{exp_fl_rank_def_plot_4}
    \end{subfigure}

     \begin{subfigure}[t]{0.45\textwidth}
    \centering
    \includegraphics[page=3, width=\linewidth]{new_plots/randn_fl.pdf}
    \caption{$d = 1000$}  \label{exp_fl_rank_def_plot_5}
    \end{subfigure}
    \hfill
    \centering
    \begin{subfigure}[t]{0.45\textwidth}
        \centering
        \includegraphics[page=3, width=\linewidth]{new_plots/randn_fl_lognorm.pdf}
    \caption{$d = 1000$} \label{exp_fl_rank_def_plot_6}
    \end{subfigure}

\caption{The matrices are initially drawn from both a standard normal and a log-norm distribution with the smallest singular value set to 0, and stochastically rounded to single precision. The horizontal axis represents the distribution of $\sigma_d(\Abtil)$ over 100 repetitions, grouped into up to 10 bins. The vertical axis shows the frequency with which each $\sigma_d(\Abtil)$ appears in each bin. The
\textcolor{orange}{orange dashed vertical line} represents the average value of $\sigma_d(\Abtil)$, while the \textcolor{red}{red dashed vertical line} represents the lower bound estimate~(\ref{eq:maineqapprox}).
Each panel corresponds to a different combination of the initial distribution and $d$. In each row, column dimension $d$ is fixed, while the distribution varies.} \label{exp_fl_rank_def_plot}

\end{figure}

\begin{figure}[ht!]
    \begin{subfigure}[t]{0.45\textwidth}
    \centering
    \includegraphics[page=1, width=\linewidth]{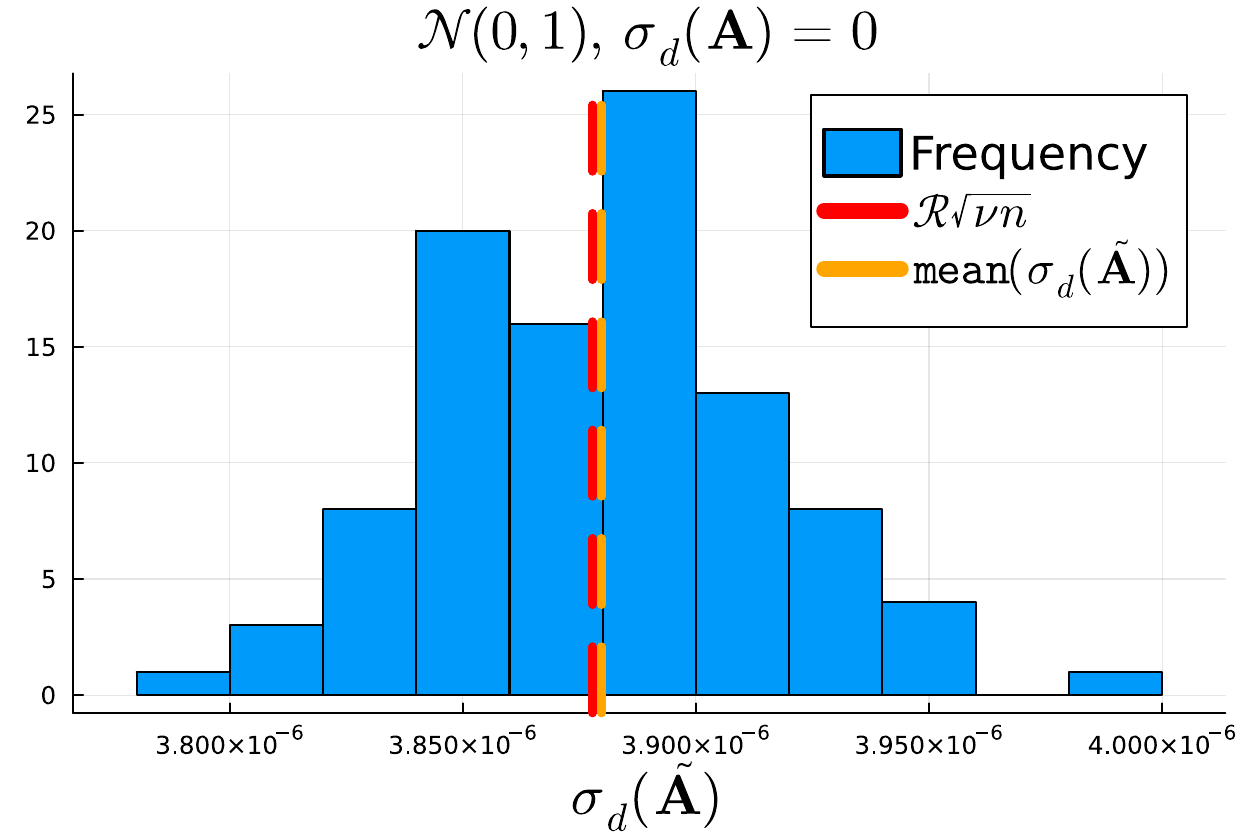}
    \caption{$\Ab^{h}, \,d = 10$} \label{exp_fl_nu_plot_1}
    \end{subfigure}
    \hfill
    \centering
    \begin{subfigure}[t]{0.45\textwidth}
        \centering
        \includegraphics[page=1, width=\linewidth]{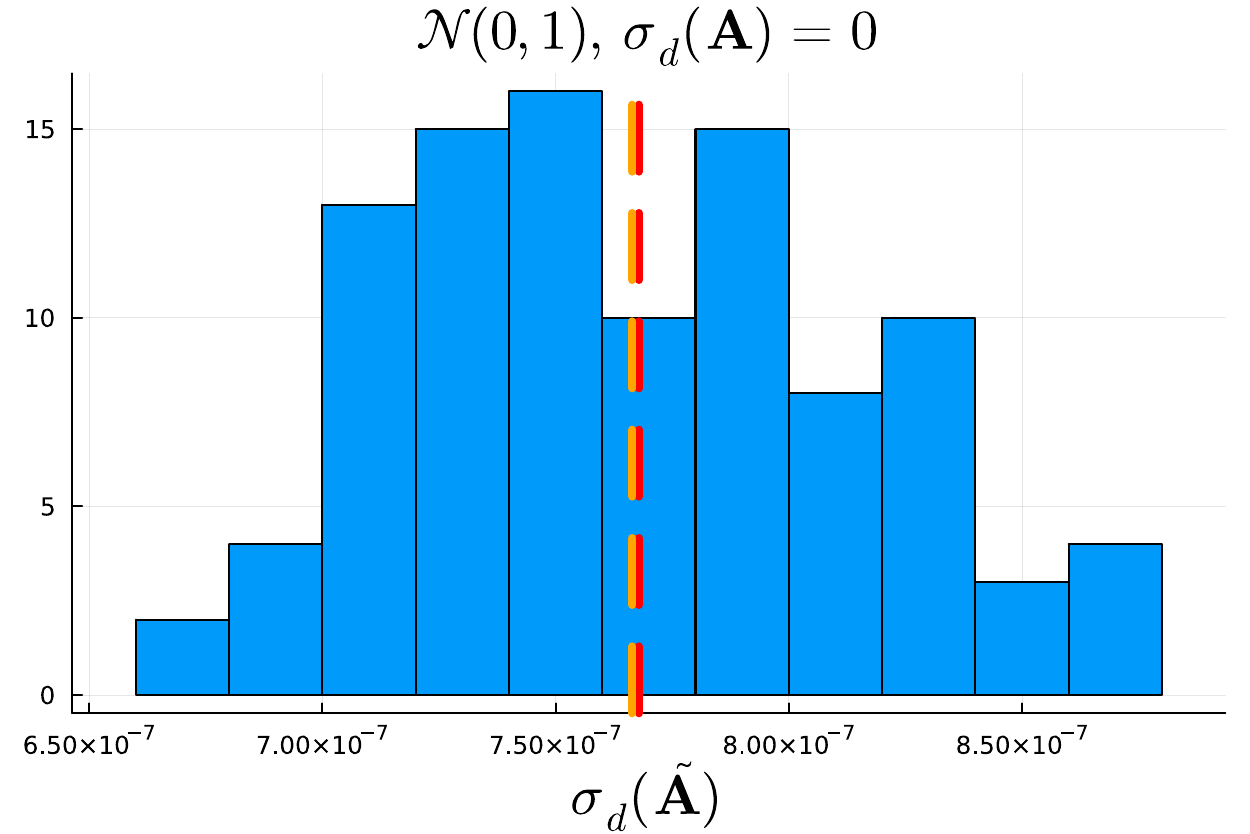}
    \caption{$\Ab^{l}, d = 10$}  \label{exp_fl_nu_plot_2}
    \end{subfigure}

    \begin{subfigure}[t]{0.45\textwidth}
    \centering
    \includegraphics[page=2, width=\linewidth]{new_plots/randn_fl_nu_h}
    \caption{$\Ab^{h}, d = 100$}  \label{exp_fl_nu_plot_3}
    \end{subfigure}
    \hfill
    \centering
    \begin{subfigure}[t]{0.45\textwidth}
        \centering
        \includegraphics[page=2, width=\linewidth]{new_plots/randn_fl_nu_l}
    \caption{$\Ab^{l}, d = 100$} \label{exp_fl_nu_plot_4}
    \end{subfigure}

     \begin{subfigure}[t]{0.45\textwidth}
    \centering
    \includegraphics[page=3, width=\linewidth]{new_plots/randn_fl_nu_h}
    \caption{$\Ab^{h}, d = 1000$}  \label{exp_fl_nu_plot_5}
    \end{subfigure}
    \hfill
    \centering
    \begin{subfigure}[t]{0.45\textwidth}
        \centering
        \includegraphics[page=3, width=\linewidth]{new_plots/randn_fl_nu_l}
    \caption{$\Ab^{l}, d = 1000$} \label{exp_fl_nu_plot_6}
    \end{subfigure}

\caption{The matrices are initially drawn from a standard normal distribution with the smallest singular value equal to 0, and stochastically rounded to single precision. The horizontal axis represents the distribution of $\sigma_d(\Abtil)$ over 100 repetitions, grouped into up to 10 bins. The vertical axis shows the frequency with which each $\sigma_d(\Abtil)$ appears in each bin. The
\textcolor{orange}{orange dashed vertical line} represents the average value of $\sigma_d(\Abtil)$, while the \textcolor{red}{red dashed vertical line} represents the lower bound estimate~(\ref{eq:maineqapprox}).
Each panel corresponds to a different combination of 'high' or 'low' $\nu$ and $d$. In each row, the column dimension $d$ is fixed, while the value of $\nu$ varies.} \label{exp_fl_nu_plot}

\end{figure}

\end{document}